\documentclass[twoside,11pt,preprint]{article}

\usepackage{blindtext}

\usepackage{jmlr2e}

\usepackage{lastpage}
\jmlrheading{xx}{2026}{1-\pageref{LastPage}}{x/xx; Revised x/xx}{x/xx}{xx-0000}{Thuot, Vogt, Ghoshdastidar, Verzelen}

\ShortHeadings{Nonparametric Kernel Clustering with Bandit Feedback}{Thuot, Vogt, Ghoshdastidar, Verzelen}
\firstpageno{1}

\usepackage{microtype}
\usepackage{booktabs}
\usepackage{times}

\usepackage{color}
 \usepackage{ dsfont }
\usepackage{mathtools}
\usepackage{bbm}
\usepackage{algorithm}
\usepackage{algorithmic}
\usepackage{xcolor}
\usepackage{multicol}

\newcommand{\Rnums}{\mathbb{R}}
\newcommand{\Nnums}{\mathbb{N}}
\newcommand{\Hspace}{\mathcal{H}}
\newcommand{\norm}[1]{\left\lVert#1\right\rVert}

\newcommand{\Ex}[2]{\mathop{\mathbb{E}}_{#1}\left[#2\right]}

\renewcommand{\P}[1]{\mathbb{P}\left(#1\right)}

\newcommand{\empdist}{\norm{\hat{\mu}_j - \hat{\mu}_i}}
\newcommand{\emptoreal}[1]{\norm{\mu_{#1} - \hat{\mu}_{#1}}}
\newcommand{\empvar}[1]{\hat{\mathcal{V}}_{#1}}
\newcommand{\event}[1]{\left\{#1\right\}}

\newcommand{\kabc}{\textup{\texttt{KABC}}}
\newcommand{\kabcsg}{\textup{\texttt{KABC-SG}}}
\newcommand{\acb}{\textup{\texttt{ACB}}}
\newcommand{\tasfw}{\textup{\texttt{TaS-FW}}}
\newcommand{\acbg}{\textup{\texttt{ABC-Gaussian}}}
\newcommand{\kkm}{\textup{\texttt{Unif-KKM}}}
\newcommand{\cluster}{\textup{\texttt{CLUSTER}}}

\begin{document}

\title{Nonparametric Kernel Clustering with Bandit Feedback}

\author{\name Victor Thuot\thanks{Equal contribution} \email victor.thuot@ens-lyon.fr \\
\addr INRAE, Mistea, Institut Agro\\
Universit{\'e} Montpellier\thanks{Present address of Victor Thuot: UMPA, École Normale Supérieure de Lyon, Lyon, France.}
\AND
\name Sebastian Vogt$^{\scriptsize *}$ \email sebastian.vogt@mgm-tp.com \\
\addr mgm technology partners
\AND
\name Debarghya Ghoshdastidar \email ghoshdas@cit.tum.de \\
\addr Technical University of Munich \\
TUM School of Computation, Information and Technology\\
Munich Centre for Machine Learning
\AND
\name Nicolas Verzelen \email nicolas.verzelen@inrae.fr \\
\addr INRAE, Mistea, Institut Agro\\
Universit{\'e} Montpellier
}

\maketitle

\begin{abstract}%
Clustering with bandit feedback refers to the problem of partitioning the arms of an $N$-armed stochastic bandit according to their underlying distributions, based on noisy observations collected sequentially and adaptively. This setting has attracted growing interest for its applicability in recommendation systems, crowdsourcing, and adaptive trials. Existing works rely on the strong assumption that arm distributions are sub-Gaussian, which can only distinguish arms whose distributions differ in mean — missing realistic settings where distributional differences lie in shape or higher moments.
We introduce a framework of \emph{nonparametric clustering with bandit feedback}, free of any parametric assumptions regarding the arm distributions. We adopt a kernel-based approach and reformulate the problem as clustering arms according to their kernel mean embeddings in a reproducing kernel Hilbert space (RKHS). Building on this formulation, we introduce the \kabc\ algorithm, prove that it is $\delta$-PAC, and derive a non-asymptotic high-probability upper bound on its sampling complexity. The bound is governed by a variance-aware signal-to-noise ratio defined as the minimum over pairs of arms in distinct clusters of the ratio between their squared MMD and their RKHS variance. \kabc\ achieves instance-dependent guarantees while being adaptive to this unknown ratio. We validate our approach empirically on synthetic and real-world datasets, demonstrating that \kabc\ succeeds in distributional clustering settings where mean-based methods fail.\footnote{The source codes are available at: \url{https://github.com/gdebarghya/kernel-active-bandit}}

\end{abstract}


\section{Introduction}\label{section:introduction}

Clustering, that is, partitioning a collection of items into groups so that items within the same group are similar while items in different groups are dissimilar, is one of the most fundamental tasks in unsupervised learning and exploratory data analysis~\citep{jain2010data}. It is routinely used to uncover latent structures in data, to summarize large collections of items, or as a preprocessing step for downstream tasks, with applications ranging from biology and medicine to marketing and recommender systems. Over the decades, a wide range of procedures has been developed, from $k$-means and its kernelized variants~\citep{dhillon2004kernel} to spectral methods~\citep{ng2001spectral,vonluxburg2007tutorial}, together with a statistical theory characterizing when an underlying partition can be recovered from data, including in nonparametric mixture models~\citep{aragam2020identifiability,vankadara2021recovery}.

In this classical \emph{batch} setting, the dataset is given beforehand, and the learner only has to process it. In many modern applications, however, an item is not described by a single noiseless feature vector but rather by an unknown distribution from which observations can only be drawn one at a time, each of them being noisy and costly to acquire. For instance, a recommender system learns the tastes of a user through their successive ratings; in crowdsourcing, the expertise of a worker is only revealed by repeatedly submitting tasks to them; in adaptive clinical trials, the responses of patients to a treatment are observed sequentially. In such situations, the learner controls the data acquisition process: it decides, based on past observations, which item to query next and when to stop. This opens the door to substantial savings compared to non-adaptive strategies that sample all items uniformly, since items that are easy to assign to their group require few observations, while the sampling effort can be concentrated on ambiguous items that are hard to distinguish from one another. The challenge is then twofold: designing sampling and stopping rules that adapt to the unknown difficulty of the problem and guaranteeing that the returned partition is correct with high probability while using as few observations as possible.

This is precisely the purpose of \emph{clustering with bandit feedback}, a pure exploration problem introduced by \cite{yang2024optimal} and \cite{thuot2024active}. In this problem, each item is modeled as an arm of a stochastic multi-armed bandit. The learner sequentially and adaptively collects noisy observations from the arms in order to recover an unknown partition of the arms, while providing guarantees on the returned partition. In both works, the observations are vectors in $\mathbb{R}^d$ with Gaussian or sub-Gaussian distributions, and two arms belong to the same group if and only if they share the same mean. In this parametric setting, \cite{yang2024optimal} characterizes the optimal expected budget in the asymptotic regime where the probability of error goes to zero, while \cite{thuot2024active} provides a non-asymptotic characterization of the optimal budget. %See \citep{yang2024optimal} for further applications.

In this work, we consider a \emph{nonparametric} version of this problem. Mean-based formulations can only separate arms whose distributions differ in mean, thereby missing settings where the differences between groups lie in the shape or higher moments of the distributions. For example, consider a crowdsourcing platform seeking to partition workers based on their specialized skills across diverse task categories. While clustering workers based on their mean performance across tasks would return a trivial separation of `good' vs. `bad' workers, distributional clustering may identify a useful partition of workers who could be assigned specific tasks.
Besides, parametric assumptions are often untenable for modern high-dimensional or complex data~\citep{muandet2017kernel}, for example, when clustering noisy biological or medical signals in adaptive trials.

As in \cite{yang2024optimal,thuot2024active}, we adopt the stochastic multi-armed bandit model \citep{lattimore2020bandit}, in which a learner interacts sequentially with an unknown environment with $N\in \mathbb{N}$ arms. Each arm $i$ is associated with an unknown distribution $\nu_i$ supported on a general space $\mathcal{X}$, on which no parametric assumption is made. At each time step, the algorithm chooses an arm and observes a sample from its associated distribution. The algorithm interacts with the environment until a (random) stopping time $\tau$, which it chooses. We assume that there exists an underlying partition of the arms into $K$ groups such that two arms are in the same group if and only if they are associated with the same distribution. The learner's objective is to recover this partition exactly with a probability of error of at most the prescribed $\delta$, in which case the algorithm is said to be \emph{$\delta$-PAC}. Since data collection is costly, the goal is to design $\delta$-PAC algorithms that make as few observations as possible.

\paragraph{Our Contribution (kernel-based formulation of clustering with bandit feedback).}
The key conceptual contribution of the present work is to map the problem of clustering nonparametric distributions in the original topological space $\mathcal{X}$ to a distributional clustering problem on a reproducing kernel Hilbert space (RKHS) $\mathcal{H}$, where the distribution of each arm $i$ is represented through its kernel mean embedding (KME). 
This reformulation relies on the simple fact that if a positive semi-definite kernel function $g$ is a bounded and characteristic kernel, then the mapping of any distribution to its KME in the corresponding RKHS is injective \citep{sriperumbudur2010hilbert}. Hence, one can quantify the distance between two distributions in terms of the RKHS distance between their KMEs or, equivalently, the maximum mean discrepancy (MMD) \citep{gretton2012kernel,muandet2017kernel}. While  kernel and MMD-based methods are well-established in nonparametric hypothesis testing \citep{gretton2012kernel} and nonparametric clustering \citep{vankadara2021recovery}, the utility of this principle has not been explored in active clustering problems.

In Section~\ref{sec:setting}, we introduce a kernel-based formulation of the nonparametric clustering with bandit feedback problem.
In Section \ref{sec:algorithm}, we introduce \emph{Kernel Active Bandit Clustering} (\kabc, Algorithm~\ref{algo:KABC}), a new sequential and adaptive algorithm for clustering with bandit feedback.
\kabc\ is designed to be \emph{$\delta$-PAC}---it outputs the correct partition of the arms with an error probability of at most $\delta$.
It leverages state-of-the-art concentration inequalities for empirical KMEs to adapt simultaneously to the unknown MMD between groups and the unknown RKHS variances (see~\eqref{eq:variance}). In Theorem~\ref{theorem:firstAdaptive}, we establish a non-asymptotic, high-probability upper bound on the budget (the number of samples) for \kabc, and we identify a variance-aware signal-to-noise ratio governing the difficulty of the problem. At a high level, \kabc\ is an adaptive algorithm that, at each iteration, runs variance-aware kernel two-sample tests based on empirical MMD and empirical variances for all pairs of arms, increasing the sampling budget until all tests can reliably decide whether two arms belong to the same group or not.
In Section~\ref{sec:numerical}, we empirically validate \kabc\ on synthetic and real-world (MovieLens) datasets, demonstrating that it succeeds in nonparametric settings where mean-based methods fail and remains competitive with parametric baselines in the Gaussian case.
Section~\ref{sec:discussion} concludes with further comments and perspectives. The proof of our main theoretical contribution can be found in Appendix~\ref{appendix:proof_UB}. Appendix~\ref{app:exp} provides more details on the implementation and calibration of the method. 

%\paragraph{Outline}
%Section~\ref{sec:setting} introduces the problem setting and notation. Section~\ref{sec:algorithm} presents the \kabc\ algorithm and our main upper bound on its sampling complexity. Section~\ref{sec:numerical} provides numerical experiments. Section~\ref{sec:discussion} concludes with further comments and perspectives. The proof of our main theoretical contribution can be found in Appendix~\ref{appendix:proof_UB}. Appendix~\ref{app:exp} provides more details on the implementation and calibration of the method. 

%\subsection{Related works}

\paragraph{Related work on clustering with bandit feedback problems.}
The clustering with bandit feedback problem (CBP) has attracted increasing attention in recent years \citep{yang2024optimal,thuot2024active,chandran2025online,yavas2025general,graf2025clustering}. The problem was first formalized by \cite{yang2024optimal}, who consider a parametric setting where arms are partitioned according to their $d$-dimensional means and the arm distributions are Gaussian. They introduce the BOC algorithm, a $\delta$-PAC procedure based on the Track-and-Stop method of \cite{garivier2016optimal}, and establish that an expected budget is asymptotically optimal in the regime $\delta \to 0$. In turn, \cite{thuot2024active} study the same parametric CBP and provide a non-asymptotic characterization of the optimal complexity. A variant of the problem is considered by~\cite{chandran2025online}, where the target partition is defined as the single-linkage clustering of the arm means. \cite{yavas2025general} study a family of distribution-matching problems that encompasses the CBP under the assumption that the arm distributions lie on a finite alphabet. In a different direction, \cite{graf2025clustering} study a similar clustering problem, but the learner is only able to sample partial information on each arm. All these works rely on strong parametric or distributional assumptions (e.g., Gaussian or sub-Gaussian arms, finite alphabets). In contrast, we consider the active clustering problem in a genuinely nonparametric setting by working in an RKHS and clustering arms using their kernel mean embeddings (KMEs).

\paragraph{Related work on kernel methods and kernel clustering.}
Recent advances in kernel methods have shown that mapping distributions to their KMEs in RKHSs provides a powerful and versatile framework for statistical learning on distributions \citep{smola2007hilbert,muandet2017kernel}. This framework has led to a wide range of applications, including kernel two-sample tests \citep{gretton2007akernel,gretton2012kernel}, independence testing \citep{gretton2007bkernel}, and many other distributional inference tasks; see \citep{muandet2017kernel,berlinet2011reproducing} for comprehensive reviews. From a more geometric point of view, \cite{sriperumbudur2010hilbert,sriperumbudur2011universality} studies when RKHS embeddings induce metrics on probability measures and characterizes universality and characteristic kernels.
Within this framework, kernel two sample tests compare empirical KMEs of two samples using the maximum mean discrepancy (MMD) as a test statistic, yielding nonparametric tests that avoid explicit density estimation. \cite{tolstikhin2017minimax} study the minimax estimation of KMEs via their empirical counterparts and show that, for bounded kernels, the optimal rate depends only on kernel properties and not on the underlying distributions. \cite{tolstikhin2016minimax} extend this perspective to the minimax estimation of the MMD. \cite{wolfer2025variance} derive state-of-the-art, variance-aware concentration bounds for the MMD between true and empirical KMEs. Our builds on this perspective, as clusters are identified through the estimation of the MMD distance between pairs of arms, and our analysis relies on the variance-aware KME and MMD concentration inequalities.

%\paragraph{Related work on Kernel clustering.}
Another related line of work is the fruitful development of kernel methods for clustering. Kernel-based methods such as kernel $K$-means~\citep{dhillon2004kernel} and kernel spectral clustering~\citep{ng2001spectral} have been popular in practice for over 20 years, specifically for finding clusters with complex geometry. While initial statistical guarantees for kernel clustering were investigated under (sub)Gaussian mixture models \citep{CouilletBG16,YanS16,VankadaraG20}, recently \cite{vankadara2021recovery} established separability conditions under which kernel-based clustering can recover the underlying true partition under nonparametric mixture models. In particular, their analysis highlights that these conditions can be expressed in terms of MMD between component distributions. 
In a complementary line, kernel $K$-means has also been proposed for clustering distributional data by applying $K$-means directly to KMEs in RKHS and using MMD as the distance between probability measures \citep{baillo2025kernel}. \cite{fleissner26a} study the interpretability of kernel clustering by introducing a notion of the price of explainability for nonparametric mixture models in terms of the KME.
These works focus on \emph{batch} clustering from a fixed sample. In contrast, we address an \emph{active}, bandit-style setting where the learner adaptively decides which distributions (arms) to sample in order to recover the clustering.

\section{Nonparametric Clustering with Bandit Feedback}\label{sec:setting}

We describe the problem of sequential and active clustering using bandit feedback (CBP) in a general nonparametric setting and reformulate the problem under a kernel-based framework.

\paragraph{Model and (unknown) partition.} 
Let $\mathcal{X}$ be a separable topological space. There are $N$ objects, each associated with an unknown distribution $\nu_i$, $i=1,\ldots, N$ on $\mathcal{X}$. 
Henceforth, $N$ is referred to as the number of arms, and we call $\nu = \{\nu_1,\dots,\nu_N\}$ the environment. 
There exists an (unknown) partition of $[N]$, referred to as $\mathcal{C}^*$, such that two arms $i$ and $j$ belong to the same group of $\mathcal{C}^*$ if and only if $\nu_i= \nu_j$. %Thus, $\mathcal{C}^*$ is the unknown partition a
We denote the number of groups or clusters in $\mathcal{C}^*$ by $K$.
We do not make any parametric assumptions about the distributions $\{\nu_i\}_{i\in[N]}$. 

\paragraph{Sequential algorithms and $\delta$-PAC objective.} We work in an adaptive, sequential setting where an algorithm interacts with the bandit environment---that is, the algorithm draws random samples from each arm---in order to recover the clustering. More specifically, the algorithm collects data sequentially: at each time $t\geq 1$, given the previous observations $((A_1,X_1),(A_2,X_2),\ldots, (A_{t-1},X_{t-1}))$, the algorithm selects an arm $A_t \in [N]$ and observes  a new random variable $X_t \sim \nu_{A_t}$. Then, at some point, the algorithm stops and outputs a partition of $[N]$. 
%
%Let us define such a sequential algorithm more formally. For that purpose, we 
Defining the filtration $\mathcal{F}_t = \sigma(A_1,X_1,\dots,A_t,X_t)$ at time $t$, a sequential algorithm can be formally described by: 
\begin{enumerate}
    \item \emph{a (possibly random) selection rule} that chooses the next arm $A_t$ based on the past $\mathcal{F}_{t-1}$;
    \item \emph{a stopping time} $\tau$ that is measurable with respect to $(\mathcal{F}_t)$; 
    \item \emph{a recommendation rule} that outputs a clustering $\mathcal{C}_\tau$ based on $\mathcal{F}_\tau$.
\end{enumerate}
Henceforth, $\tau$ is referred to as the budget of the procedure; this corresponds to the total number of observations required by the algorithm.  We say $\mathcal{C}_\tau$ is correct if it is equal to the true underlying partition $\mathcal{C}^*$.
Given a confidence level $\delta \in (0,1]$ and a class of environments $\mathcal{E}$---that is, a class $\mathcal{E}$ of distributions---a sequential algorithm is   $\delta$-PAC if
\[
    \forall \nu \in \mathcal{E} \colon\quad
    \P{\mathcal{C}_\tau \text{ is correct}} \;\ge\; 1-\delta \enspace . 
\]
In other words, $\mathcal{C}_\tau$ recovers the true partition with high probability, this uniformly over all distributions $\mathcal{E}$ such that $\mathcal{C}^*$ contains $K$ groups. Our objective is to design $\delta$-PAC algorithms whose sampling budget is as small as possible. Throughout the main sections, we assume that the number $K$ of groups is known, and our objective is to recover exactly $\mathcal{C}^*$, but we discuss some extensions in Section~\ref{sec:discussion}.

\paragraph{Reformulation of CBP in terms of kernel mean embedding.}
Let $g : \mathcal{X}  \times \mathcal{X} \to \Rnums$ be a continuous, positive definite kernel on $\mathcal{X}$. The kernel $g$ induces a reproducing kernel Hilbert space (RKHS) $(\Hspace, \langle \cdot, \cdot \rangle)$ and a feature map
\(\phi : \mathcal{X} \to \Hspace\) such that, for every $x, y \in \mathcal{X}$, $g(x, y) = \langle \phi(x), \phi(y)$.
%\begin{equation}\label{eq:kerneldot}
%    \forall x, y \in \mathcal{X} \colon\quad
%    g(x, y) = \langle \phi(x), \phi(y) \rangle \enspace.\end{equation}
%
We assume that $g$ is bounded in the sense that $\sup_{x \in \mathcal{X}} g(x,x) < +\infty$,
and we define its supremum $\bar{g} \coloneqq \sup\limits_{x, y \in \mathcal{X}} g(x, y)$ and range $\tilde{g} \coloneqq \sup\limits_{x, y \in \mathcal{X}} g(x, y) - \inf_{x, y \in \mathcal{X}} g(x, y)$.
%
%We define the supremum and range of $g$ as
%\begin{multicols}{2}
%\noindent
%\begin{equation*}
%    \bar{g} \coloneqq \sup_{x, y \in \mathcal{X}} g(x, y) \enspace,
%\end{equation*}\columnbreak
%\begin{equation*}
%    \tilde{g} \coloneqq \sup_{x, y \in \mathcal{X}} g(x, y) - \inf_{x, y \in \mathcal{X}} g(x, y) \enspace.
%\end{equation*}
%\end{multicols}
%
We make two further assumptions on $g$ that it is \emph{translation invariant} and \emph{characteristic kernel}, as defined below. We say $g$ is {translation invariant} if there exists a function
$\Psi : \mathcal{X} \to \Rnums$ such that
$g(x, x') = \Psi(x - x') \;\forall x, x' \in \mathcal{X}$.
For a distribution $\nu$ on $\mathcal{X}$, its kernel mean embedding (KME) with respect to the kernel $g$ is defined as $\mu_{\nu} \coloneqq \Ex{X \sim \nu}{\phi(X)} \in \Hspace$,  where $\mu_{\nu}$ is  a Bochner integral \citep{berlinet2011reproducing}, and we call $g$ {characteristic} if the KME map
$\mu : \nu \mapsto \mu_{\nu}$ is injective. 
We refer to \citet{muandet2017kernel} for further background on KMEs and the stated assumptions, but note that popular kernels, such as Gaussian or Laplace kernels, are bounded, characteristic, and translation-invariant.

We now restate the aforementioned nonparametric CBP in terms of KMEs. For a characteristic kernel $g$, the distribution $\mu$ of a random variable $X$ is fully characterized by its KME. Therefore, testing whether two distributions $\nu_1,\nu_2$ are identical is equivalent to testing the equality of the respective KMEs: $\mu_{\nu_1} \overset{?}{=}\mu_{\nu_2}$. 
In the context of CBP, recovering the latent partition $\mathcal{C}^*$ of the distributions $\{\nu_i\}_{i\in[N]}$, corresponding to the $N$ arms, reduces to the problem of clustering their respective KMEs defined by
%\begin{equation}\label{def:KME}
%    \mu_i:=\Ex{X \sim \nu_i}{\phi(X)} \in \Hspace \enspace ,\end{equation} 
$\mu_i:=\Ex{X \sim \nu_i}{\phi(X)} \in \Hspace$ for $i=1,\ldots, N$. 

%\paragraph{Complexity parameters.}
We characterize the difficulty of the problem in terms of three main factors: the confidence parameter $\delta$, the kernel $g$, and the environment $\nu$. 
Using $\|\cdot\|$ to denote the induced norm on RKHS $\mathcal{H}$, we quantify the complexity in terms of the bound on the kernel $\bar{g}$ and distribution-dependent terms: squared maximum mean discrepancy $\textrm{MMD}^2_{i,j}$ and RKHS variance $\mathcal{V}_i^*$,
\begin{equation}\label{eq:variance}
    \textrm{MMD}^2_{i,j}:=\|\mu_i - \mu_j\|^2 \qquad \text{and} \qquad \mathcal{V}_i^* \coloneqq \Ex{X \sim \nu_i}{\|\phi(X) - \mu_i\|^2}
    \enspace,
\end{equation}
respectively, where $\textrm{MMD}^2_{i,j}$ measures the separation between two arms, and the RKHS variance characterizes the spread of $\nu_i$; that is, the noise level for arm $i$. Both can be estimated from samples.

\section{Kernel Active Bandit Clustering (\kabc)}\label{sec:algorithm}

We address the nonparametric CBP by proposing the Kernel Active Bandit Clustering (\kabc) algorithm, which is an adaptive algorithm that recovers the $K$ true clusters defined by the equality of kernel mean embeddings. The algorithm does not require any environment-specific quantity as input, but it requires the number of clusters $K$. We first describe the algorithm (pseudo-code in Algorithm~\ref{algo:KABC}) before deriving bounds on the total budget used by the procedure. At its core, \kabc\ relies on a subroutine \cluster$(n,\delta')$ (Algorithm~\ref{alg:cluster}) that takes a fixed per-arm sampling budget $n$ and a confidence parameter $\delta'\in(0,1]$. We describe \cluster\ before detailing \kabc.

\begin{figure}[b]
    \centering
    % First Algorithm Minipage
    \begin{minipage}[t]{0.34\textwidth}
        \centering
        \linespread{1.038}\selectfont
        \begin{algorithm}[H]
  \caption{: \textbf{\kabc}$(\delta, K)$}
  \label{algo:KABC}
  \begin{algorithmic}[1]
    \FOR{$k = 1,2,\ldots$}
      \STATE $\delta_k \leftarrow \frac{\delta}{4k^2}$ 
      \\ $~$\hfill\textcolor{cyan}{// Per-iteration error}
      \vspace{3pt}
      \STATE $n_k \leftarrow \left\lceil 2^k \log\frac{8(N^2 - N)}{\delta_k} \right\rceil$ \label{line:exp}
      \\ \textcolor{cyan}{// Per-arm sampling budget}%
      \STATE $\mathcal{C}_k \leftarrow \cluster(n_k,\delta_k)$
    \IF{$|\mathcal{C}_k| = K$} \label{line:checkK}
    \STATE \textbf{return} $\mathcal{C}_k$ \hfill
    %\textcolor{cyan}{// Stop when $K$ clusters are identified}
\ENDIF
    \ENDFOR
  \end{algorithmic}
        \end{algorithm}
    \end{minipage}
    \hfill
    % Second Algorithm Minipage
    \begin{minipage}[t]{0.62\textwidth}
        \centering
        \begin{algorithm}[H]
  \caption{: \textbf{\cluster}($n,\delta'$) \hfill \textcolor{cyan}{// Fixed budget $N \times n$}}
  \label{alg:cluster}
  \begin{algorithmic}[1]
    \STATE For each $i
\in[N]$, draw $n$ samples from arm $i$ \label{line:sample} 
    \STATE Compute $\big\{\empdist,\hat{\mathcal{V}}_i\big\}_{i, j \in [N], j\neq i}$ using Eqn.~\eqref{def:empdist}--\eqref{def:empvar} 
    \STATE Compute thresholds $\left\{\mathcal{B}^{i,j}(n,\delta')\right\}_{i\neq j \in [N]}$ using Eqn.~\eqref{def:threshold} \label{line:Bk}
    \STATE Define empty graph $(V,E)$ with $V =[N]$, $E= \emptyset$ \label{line:graph1}
     \FOR{$i,j \in [N], i \neq j$}
      \IF{$\empdist \leq \mathcal{B}^{i,j}(n,\delta')$}
        \STATE Add edge $(i, j)$ to $E$
         \hfill \textcolor{cyan}{// if $H_0^{i,j}$ is not rejected}
      \ENDIF
    \ENDFOR
    \STATE \textbf{return} Connected components of graph $(V,E)$ as clusters% 
    \label{line:graph2}
  \end{algorithmic}
  \end{algorithm}
    \end{minipage}
    \label{fig:algos}
\end{figure}

\noindent\textit{Overview of \cluster\ (fixed budget subroutine).}
 Given $n$ samples per arm, \cluster\ simultaneously performs kernel two-sample tests between all $\binom{N}{2}$ pairs of arms and aggregates the results via a graph-based construction. 
A graph $([N],E)$ is built on the $N$ arms, with edge $(i,j)\in E$ whenever the test does not reject the null hypothesis that arms $i$ and $j$ share the same KME. The output clustering $\mathcal{C}$ corresponds to the connected components.
This graph-based construction provides natural error control. Type-I errors (false splits: arms with identical KME placed in different clusters) are controlled with high probability by the construction of the test procedure. We guaranty that all arms with the same KMEs are clustered together, with a probability of error smaller than $\delta'$. Type-II errors (false merges: arms with distinct KME placed in the same cluster) are controlled once $n$ is large enough relative to the MMD and RKHS variances. Thus, with high probability, \cluster\ either returns the true partition or returns fewer than $K$ clusters, never creating spurious extra clusters.

\noindent\textit{Overview of \kabc\ (adaptive procedure).}
Since the minimal per-arm budget $n$ ensuring reliable clustering depends on the unknown MMD between clusters and RKHS variances $\mathcal{V}_i^*$, \kabc\
(Algorithm~\ref{algo:KABC}) applies the doubling trick. At iteration $k\ge1$, the per-arm sampling budget is
\[
    n_k = \Big\lceil 2^k\log\!\Big(\frac{8(N^2-N)}{\delta_k}\Big)\Big\rceil
    \enspace, \quad
    \delta_k = \frac{\delta}{4k^2} \enspace.
\]
The algorithm calls \cluster$(n_k,\delta_k)$ and stops as soon as $|\mathcal{C}_k|=K$, then outputs $\mathcal{C}_k$. The number of clusters $K$
is known to the learner and is used only in the stopping condition. The choice $\delta_k = \delta/(4k^2)$ ensures $\sum_k \delta_k \le \delta$, so the
overall probability of error is controlled across all iterations.

\paragraph{Kernel two-sample testing in \cluster\ and guarantees on total budget.}
A key step of the procedure is that, for each pair of distinct arms $i,j \in [N]$, the \cluster\ algorithm tests the null hypothesis $H_0^{i,j}\colon \mu_i = \mu_j$ against $H_1^{i,j}\colon \mu_i
\neq \mu_j$. Given a per-arm budget $n\geq 1$, it draws i.i.d. samples
$(X^i_1,\dots,X^i_n)\sim\nu_i$ and $(X^j_1,\dots,X^j_n)\sim\nu_j$, and computes the empirical KMEs
\[
    \hat{\mu}_i \coloneqq \frac{1}{n}\sum_{t=1}^{n} g(X^i_t,\cdot)
    \qquad \text{and} \qquad
    \hat{\mu}_j \coloneqq \frac{1}{n}\sum_{t=1}^{n} g(X^j_t,\cdot)
    \enspace.
\]
Subsequently, the empirical MMD $\empdist$ is used as the test statistic, which can be computed directly from kernel evaluations via the reproducing property \citep{gretton2012kernel}:
\begin{equation}\label{def:empdist}
    \empdist^2
    = \frac{1}{n^2}\sum_{s,t=1}^{n}
      \bigl( g(X_s^i, X_t^i) - 2 g(X_s^i, X_t^j) + g(X_s^j, X_t^j)\bigr)
      \enspace. 
\end{equation}
As $n\to\infty$, $\empdist\to 0$ under $H_0^{i,j}$ and
$\empdist\to\|\mu_i-\mu_j\|>0$ under $H_1^{i,j}$, confirming
that $\empdist$ is a consistent test statistic.
However, to specify a finite sample rejection threshold, we follow \citet{wolfer2025variance} and
estimate the RKHS variance~\eqref{eq:variance} of each arm empirically:
\begin{equation}\label{def:empvar}
    \empvar{i} \coloneqq \frac{1}{n-1} \sum_{t=1}^{n}\!\left( g(X^i_t, X^i_t)
    - \frac{1}{n}\sum_{s=1}^{n}g(X^i_t, X^i_s)\right). 
\end{equation}
The test rejects $H_0^{i,j}$ whenever
$\empdist > \mathcal{B}^{i,j}(n,\delta')$, where the variance-aware threshold
\begin{equation}\label{def:threshold}
    \mathcal{B}^{i,j}(n,\delta')=\!\left(\sqrt{\hat{\mathcal{V}}_i} +
    \sqrt{\hat{\mathcal{V}}_j}\right)\!\sqrt{\frac{2\log\frac{8(N^2-N)}{\delta'}}{n}}
    + \frac{32}{3}\sqrt{\tilde{g}}\,\frac{\log\frac{8(N^2-N)}{\delta'}}{n} 
\end{equation}
satisfies $\bigl|\,\empdist - \|\mu_i-\mu_j\|\,\bigr| \le \mathcal{B}^{i,j}(n,\delta')$
with a probability of at least $1-\delta'/(N^2-N)$ (see Lemma~\ref{lemma:varaw}). Combining the thresholds of the above two-sample tests with the overall pipeline of Algorithms \ref{algo:KABC}--\ref{alg:cluster}, we provide the following bound on the total budget of \kabc\ (proof in Appendix~\ref{appendix:proof_UB}).

\begin{theorem}[\kabc$(\delta, K)$ is $\delta$-PAC]\label{theorem:firstAdaptive}
Let $\delta \in (0,1]$ and $g$ be a continuous, positive definite, characteristic, translation invariant, bounded kernel.
Define the variance-aware signal-to-noise ratio
\begin{equation}\label{def:SNR}
    s_*^2 := \min_{\substack{i < j \in [N] \\ \mu_i \neq \mu_j}}
    \left(
        \frac{\|\mu_i - \mu_j\|^2}{\mathcal{V}_i^*\vee\mathcal{V}_j^*}
        \wedge
        \frac{2 \|\mu_i - \mu_j\|}{\sqrt{\bar{g}}}
    \right) \enspace. 
\end{equation}
Then Algorithm~\ref{algo:KABC} \kabc$(\delta, K)$ is $\delta$-PAC, that is, it outputs the true clustering with probability at least $1-\delta$ for all possible environments with $K$ groups.

Furthermore, with probability at least $1 - \delta$, the total budget $\tau$ of \kabc$(\delta, K)$ satisfies
    \begin{equation}\label{eq:bound_budget}
        \tau \leqslant 8N \cdot\left(\frac{128}{s_*^2}\vee 1\right)\cdot \left(\log\left(\frac{32(N^2 - N)k_*^2}{\delta} \right)\right), 
    \end{equation}
where $k_* =\left\lceil \log_2\left(\frac{128}{s_*^2}\right)\right\rceil\vee 1 $ is a logarithmic term independent of $\delta$. 
\end{theorem}
% 
%\paragraph{Comments on the budget.}
The bound in Eqn.~\eqref{eq:bound_budget} decomposes into two factors: an overall factor of $N$ arising from the uniform sampling strategy, and a per-arm complexity of order $s_*^{-2}\log(N/\delta)$ governed by the signal-to-noise ratio $s_*$. 
The factor $N$ reflects the fact that every call to \textsc{Cluster} (Algorithm~\ref{alg:cluster}) queries each arm the same number of times. When all inter-cluster separations are of the same order, such a linear dependence on $N$ is unavoidable.
The per-arm budget $s_*^{-2}\log(N/\delta)$ matches the sample complexity of a single nonparametric two-sample test based on the MMD, as achieved by state-of-the-art procedures such as \citet{wolfer2025variance}. 
The factor $s_*^{-2}$ captures the difficulty of the hardest pair of arms to separate, since one must test $H_0^{i,j}\colon \mu_i = \mu_j$ for every pair $(i,j)$ belonging to distinct clusters. The logarithmic factor 
$\log(N^2/\delta)$ arises from a union bound over all $\binom{N}{2}$ pairs, together with the requirement of exact recovery of the full partition.
Finally, the cost of adapting to the unknown value of $s_*$ introduces an additive doubly-logarithmic term of order $c\,\frac{N}{s_*^{2}}\log\log(s_*^{-2})$, which also appears inside the logarithm of~\eqref{eq:bound_budget}. This is the price paid by the doubling trick in Algorithm~\ref{algo:KABC} for 
identifying the right stopping iteration without prior knowledge of $s_*$.

\paragraph{Computational complexity.}
The computational complexity of \kabc\ is  $O\left(\big(N+ \frac{1}{s_*^2}\big)N\log(N)\right)$.  We could slightly improve it to  $O\left(\big(K+ \frac{1}{s_*^2}\big)N\log(N)\right)$ by evaluating $\mathbf{1}\{\empdist \leq \mathcal{B}^{ij}\}$ more efficiently, as done in \citet{thuot2024active}. Here, we choose to keep the description of \kabc\ simple.

\paragraph{Variance-aware vs sub-Gaussian bound and \kabcsg\ algorithm}
\kabc\ and Theorem \ref{theorem:firstAdaptive} rely on the variance-aware bounds of \citet{wolfer2025variance}. As an alternative, one could use sub-Gaussian-type bounds; for instance, \citet[Proposition~A.1]{tolstikhin2017minimax}. When using a sub-Gaussian bound, one would only replace the threshold $\mathcal{B}^{i,j}(n,\delta')$ in Algorithm~\ref{alg:cluster} with \(\sqrt{\frac{\bar{g}}{n}}\left(\sqrt{\log(\frac{8(N^2-N)}{\delta})}+2\right)\). The modified algorithm is still $\delta$-PAC, but its budget $\tau$ is bounded by $\max\limits_{i\neq j}\frac{N\bar{g}}{\|\mu_i-\mu_j\|^2}\log(\frac{N}{\delta})$. Since \(\mathcal{V}_i^*\leqslant \bar{g}\) for every arm \(i\), the variance-aware bounds of \citet{wolfer2025variance} are never worse and can substantially improve the budget whenever \(\mathcal{V}_i^* \ll \bar{g}\).
In the numerical experiment, we also consider this sub-Gaussian-bound variant of the \kabc\ algorithm, referred to as \kabcsg, and verify that variance-aware bounds lead to better performance.

\section{Numerical experiments}\label{sec:numerical}

We empirically evaluate the performance of the proposed \kabc\ (Algorithm \ref{algo:KABC}) and its sub-Gaussian variant (\kabcsg) with two baseline approaches: the \acb\ algorithm proposed in \citep{thuot2024active}, which is designed for clustering Gaussian mixtures, and the \tasfw\ algorithm proposed in \citep{yavas2025general}, which is designed to actively cluster distributions over a finite alphabet.
%We first consider a setting of Gaussian mixture clustering and show that \kabc\ performs comparably to the \acb\ algorithm, which is provably optimal for this setting. Subsequently, we consider some non-Gaussian settings where we show that if the clusters are distributionally different but have identical means, then \acb\ fails completely, but \kabc\ recovers the clusters perfectly. Finally, we consider a semi-synthetic setup of clustering movies from the MovieLens dataset based on samples of their ratings. In this setting, \kabc\ and \tasfw\ perform well, but \kabc\ is significantly more scalable than \tasfw\ and can handle problems with more than two clusters. We also compare the active clustering algorithm with a passive baseline of kernel $K$-means \citep{dhillon2004kernel} with uniformly drawn samples of ratings across the arms (\kkm).
We describe the experimental design and observations below and illustrate the performance of the methods in Figures \ref{fig:exp-gmm}--\ref{fig:exp-movielens}. All experiments are averaged over 100 independent runs, and the kernel algorithms use the Gaussian kernel, where the bandwidth is chosen using a median heuristic.
Our implementation of all algorithms is calibrated in the following sense. The budget used in \kabc, as specified by $\mathcal{B}^{i,j}(n,\delta')$ in \eqref{def:threshold}, is conservative in practice. We use a constant factor to scale down the $\mathcal{B}^{i,j}(n,\delta')$ obtained from \eqref{def:threshold}. The scaling factor is obtained via calibration in a data-agnostic manner. Calibration constants are also used to implement \acb\ and to improve the conservative budget of \tasfw\ and \kabcsg. Appendix \ref{app:exp} reports the calibration constants and ablation studies.

\begin{figure}[t]
  \centering
  \includegraphics[width=\textwidth,clip=true,trim=0mm 0mm 0mm 0mm]{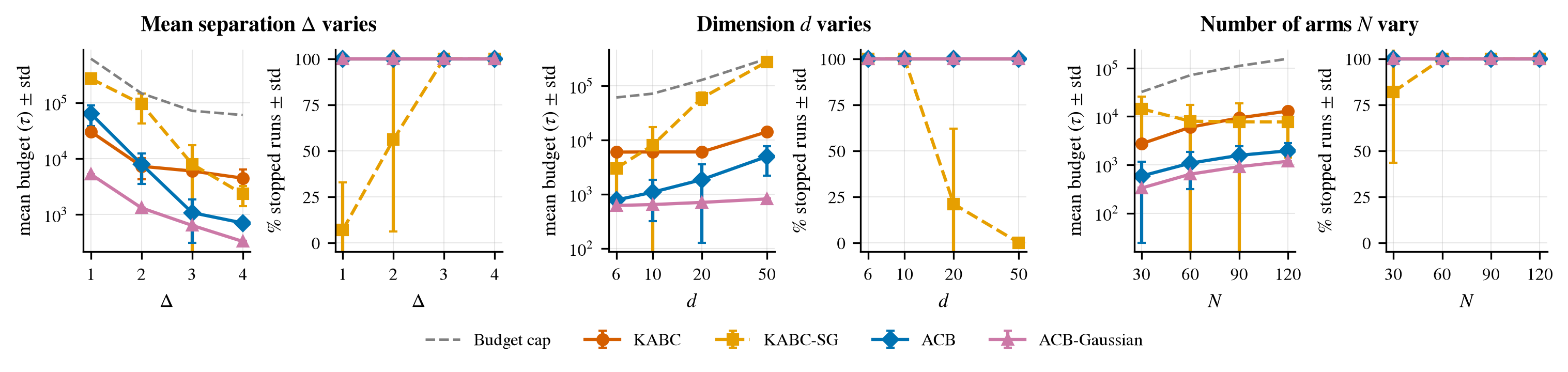}
  \vskip-1ex
  \caption{Performance of different algorithms for clustering arms with isotropic Gaussian distributions into $K=3$ clusters. The plots show how the budget used by each algorithm and the proportion of runs that converge vary with the separation $\Delta$, data dimension $d$ and number of arms $N$. The mean and standard deviations are shown, along with a simplified version of the theoretical budget limit \eqref{eq:bound_budget} used in our implementation.}
  \label{fig:exp-gmm}
\end{figure}

\paragraph{Experiment 1: Gaussian mixture clustering.}
We first conduct numerical studies under the setup considered in \citep{thuot2024active}, which assumes the samples are drawn from a Gaussian mixture model in $\mathbb{R}^d$. Specifically, we consider $K=3$ distinct isotropic Gaussian distributions $\mathcal{N}(\mu_j,I_{d\times d})$, $j\in\{1,2,3\}$, where the means $\mu_1,\mu_2,\mu_3 \in\mathbb{R}^d$ are such that $\Vert \mu_j-\mu_{j'}\Vert = \Delta$ for every pair $j\neq j'$. The sampling distributions $\{\nu_i\}_{i\in[N]}$ of the $N$ arms correspond to one of the $K=3$ Gaussians, and we assume the clusters are balanced, that is, exactly $N/K$ arms have mean $\mu_i$.

We consider a base setting of $N=60, d=10, \Delta=3$ and individually vary $\Delta \in \{1,2,3,4\}$, $d \in\{6, 10, 20, 50\}$ or $N \in \{30, 60, 90, 120\}$ to study the performance of the algorithms in terms of their budget and the proportion of converged runs.
Figure \ref{fig:exp-gmm} compares the \kabc, \kabcsg, and \acb\ algorithms. It also includes \acbg, which is a non-adaptive variant of \acb\ (see Appendix of \citep{thuot2024active}) and uses an analytically computed threshold applicable for Gaussian samples with known $\Delta$. The plots also show the ``theoretical'' allowable budget fixed in our implementations, set at $\tau = \frac{200N}{s_*^2}\log\left(\frac{N}{\delta}\right)$, which is inspired by \eqref{eq:bound_budget} but with smaller constants. 
The plots validate that all algorithms typically work for this problem, with \kabc, \acb, and \acbg\ recovering almost perfect clusters of arms in all studied scenarios. As expected, \acbg, which uses analytically derived thresholds, requires the fewest actively drawn samples, and \kabc\ is competitive with \acb. The variant with sub-Gaussian bounds, \kabcsg, requires more samples than \kabc, thereby empirically validating our claim on variance-aware vs. sub-Gaussian bounds. The calibrated version of \kabcsg\ also fails in the difficult cases of close means (small $\Delta$) or high dimensions.

\paragraph{Experiment 2: Nonparametric mixture clustering (identical means).}
To investigate the utility of a nonparametric (kernel-based) approach, we consider a few nonparametric clustering problems in one dimension. Figure \ref{fig:exp-nonparam} illustrates the two problems, each with $K=2$ clusters and $N=40$ arms equally split between the two clusters. Problem 1 assumes that two distributions are Gaussian with identical means, $\mathcal{N}(0,1)$ and $\mathcal{N}(0,4)$, while Problem 2 considers the distributions $\mathcal{N}(0,1)$ and $\frac12\mathcal{N}(-0.9,0.19)+\frac12\mathcal{N}(0.9,0.19)$, i.e., two distinct distributions with identical means and variances (Appendix \ref{app:exp} also presents results for a 2-dimensional problem).

\begin{figure}[t]
  \centering
  \includegraphics[width=\textwidth,clip=true,trim=0mm 0mm 0mm 0mm]{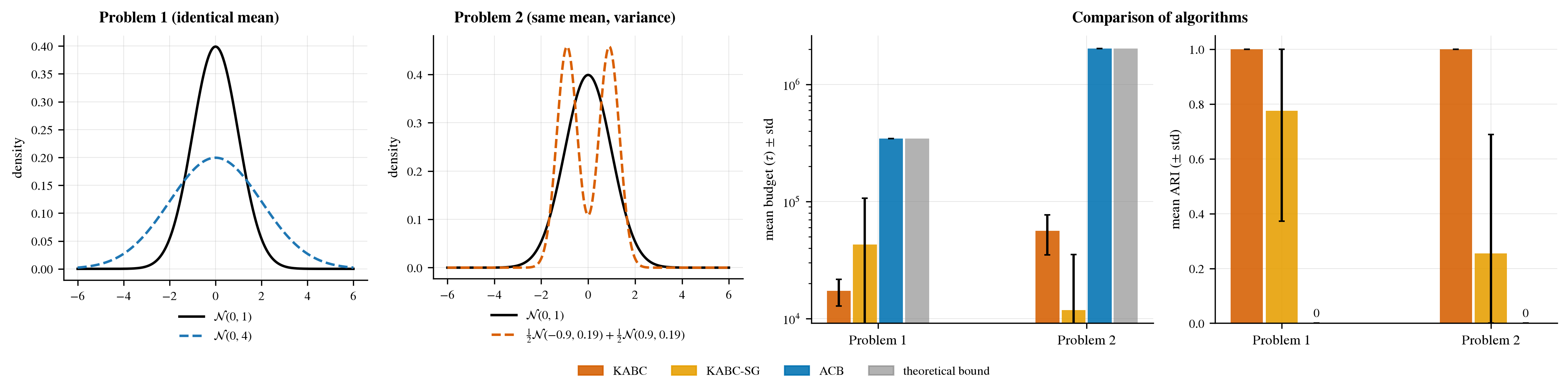}
  \vskip-1ex
  \caption{The plots on the left visualize two pairs of one-dimensional distributions used to define nonparametric clustering problems with $K=2$ clusters. The bar plots compare the
  performance of the algorithms in terms of the number of samples used and the Adjusted Rand Index (ARI) achieved in terms of mean and standard deviation over 100 runs.}
  \label{fig:exp-nonparam}
\end{figure}

Figure \ref{fig:exp-nonparam} compares \kabc, \kabcsg, and \acb\ in terms of their budget and the Adjusted Rand Index (ARI) of the recovered clustering of arms. \kabc\ exactly recovers the clusters (ARI = 1) in all cases, with a mean budget much smaller than the limit set in our experiments. On the other hand, \acb\ always reaches the limit without finding two clusters. In particular, \acb\ always has ARI = 0, as it only finds one mean and assigns all arms to a single cluster. This highlights the benefit of a nonparametric method like \kabc, which exploits separation in distributions in terms of MMD over a mean-based clustering algorithm like \acb. The performance of \kabcsg\ is less predictable; for example, it stops with few samples in Problem 2 but finds incorrect clusters. We note that \acbg\ is not applicable in this nonparametric setting since the mean separation $\Delta$ is not known a priori (also $\Delta=0$ in this case, resulting in unbounded sample complexity).

\paragraph{Experiment 3: Clustering movies from MovieLens dataset based on rating distributions.}
We now consider an experimental setup designed using the real-world movie ratings dataset, MovieLens-1M \citep{HarperK16}. The goal of this study is to analyze the efficacy of \kabc\ under more realistic distributions and its merits over a method like \tasfw\ \citep{yavas2025general} that is designed for clustering discrete distributions with bandit feedback.
The MovieLens dataset consists of ratings on a scale of 1 to 5 for various movies; that is, each movie corresponds to a probability mass function on
\{1, 2, 3, 4, 5\}. We reduce the dataset to only 207 movies with over 1000 ratings.

\noindent\emph{Two-cluster problem:}
We first consider a small-scale problem with only a pair of movies $(K=2)$ and a total of $N=8$ arms, with 4 arms corresponding to each movie. 
When an arm is sampled, one draws a single rating uniformly from the available ratings of the respective movie; hence, each arm has the rating distribution of the respective movie. The goal is to cluster the arms that correspond to the same movie, and the performance is measured in terms of the ARI of the obtained clustering and the average CPU time taken per run, as shown in Figure \ref{fig:exp-movielens}.
In each run, a pair of movies is randomly chosen such that their mean ratings differ by at most 1.0, which makes it difficult to distinguish the arms by simply estimating the means. We compare the \kabc, \acb, and \tasfw\ algorithms, each with an allowable budget of 30,000 random draws (set heuristically to make the problem realistic). For each run, we also compare with a passive kernel $K$-means baseline (\kkm), which uses the same number of samples as \kabc, but drawn uniformly at random over all arms. The \kkm\ is used to study the role of active bandit feedback over uniformly drawn samples.

In Figure \ref{fig:exp-movielens}, the left panel shows a random pair of movies with a small mean separation $(\Delta=0.01)$ but a modest squared MMD distance $(\mathrm{MMD}^2=0.011)$. The second panel compares the ARI from different algorithms, showing that \tasfw\ performs very well, followed by \kabc\ and \kkm; however, \acb\ achieves a low ARI. In particular, \acb\ often reaches the budget limit but returns a single cluster in these cases. \tasfw\ uses an exponential-time combinatorial search to find the optimal 2-way clustering of the arms, which leads to $100\times$ runtime over other methods, showing that the method is not scalable to moderate-sized clustering problems with $K>2$ clusters or more arms.

\begin{figure}[t]
  \centering
  \includegraphics[width=\textwidth,clip=true,trim=0mm 0mm 0mm 0mm]{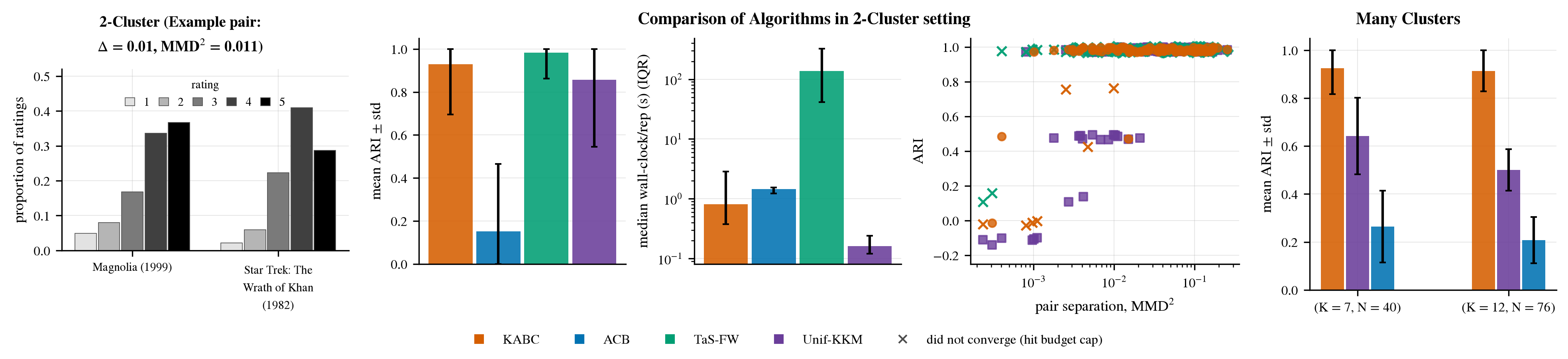}
  \vskip-1ex
  \caption{From left, the first plot shows an example pair of movies from the MovieLens-1M dataset rating used to define $K=2$ cluster problem. The second and third plots compare different clustering algorithms on the 2-cluster problem in terms of the ARI achieved and the runtime (wall clock time per run). The fourth plot shows how the ARI changes for the algorithms as the MMD between the distributions varies ($\times$ indicates run that did not stop within budget cap). The fifth plot compares the ARI of different algorithms in problems with $K\in\{7,12\}$ unbalanced clusters. Note that \kkm\ is a batch clustering algorithm. }
  \label{fig:exp-movielens}
\end{figure}

There are two notable observations. Firstly, the fourth panel shows the ARI achieved by \kabc, \tasfw, and \kkm\ in each of the 100 runs, plotted against the $\textrm{MMD}^2$ distance of the pair of distributions of ratings. For a low MMD separation, \kabc\ sometimes fails to stop within the budget limit (marked by $\times$) and achieves a low ARI, and occasionally converges to incorrect clusters for very close pairs $(\textrm{MMD}^2 \ll 10^{-3})$. In contrast, \tasfw\ fails to stop more often than \kabc, but achieves a high ARI (shown by green $\times$ with ARI=1). We believe this is part of the design of \tasfw, which runs a combinatorial search after each draw, thereby always maintaining good clusters even though the algorithm is not confident of its predictions and does not stop. Thus, \tasfw\ trades off computational time to maintain good clustering.
The second observation is that, for the simple $K=2$ cluster problem, \kkm\ achieves accuracy close to \kabc\ while requiring less runtime. This is not surprising since (i) \kkm\ does not need to adaptively learn the budget and performs a single round of kernel $K$-means using the budget learned by \kabc; and (ii) there are only two balanced clusters with a few arms, and sampling all arms identically provides reasonable performance. We show below that uniform sampling performs much worse for unbalanced clusters.

\noindent\textit{Problems with many unbalanced clusters:}
We now consider more realistic configurations with more movies/clusters $K \in \{7, 12\}$. The movies are randomly chosen with the constraints that there is at least one pair of movies with close mean ratings, at least one pair of movies with well-separated mean ratings, and the clusters are unbalanced, with the smallest cluster having 2 arms and the largest cluster having 12 arms (for $K=7, N= 40$) or 20 arms (for $K=12, N=76$).
We set a budget limit of $4\times10^5$ and $8\times10^5$ in the two settings, respectively.
Note that \tasfw\ is not scalable to these problems, but we compare the ARI achieved by \kabc, \acb, and \kkm.
The right panel of Figure \ref{fig:exp-movielens} shows that \kabc\ still achieves reasonable performance at this scale (mean ARI over 0.9), but \kkm\ performs much worse as the clusters become more unbalanced. \acb\ often fails to converge and achieves a low ARI; however, interestingly, \acb's ARI is higher than that of the $K=2$ case because here, the algorithm typically finds two or three clusters (out of 7 or 12).

\section{Discussion}\label{sec:discussion}

We study a nonparametric generalization of BCP \citep{yang2024optimal,thuot2024active}, where samples are sequentially drawn to cluster distinct nonparametric distributions. 
We address the problem using a kernel-based perspective, introducing the adaptive \kabc\ algorithm and also presenting a bound on the total budget (Theorem \ref{theorem:firstAdaptive}). Empirical results show the utility of this approach for clustering both distributions on a finite alphabet and nonparametric continuous distributions.
We conclude with some discussions on the relevance of nonparametric BCP and comments on the optimality of \kabc.

\paragraph{Resolving non-identifiability in nonparametric clustering with bandit feedback}
In a (passive) sample-level nonparametric distribution clustering problem, one typically encounters non-identifiability of the clusters, thereby making it impossible to study (almost) exact \citep[see][]{aragam2020identifiability,vankadara2021recovery}. As an example, consider the pairs of distributions in Experiment 2, where it is statistically impossible to assign a single sample to one of the distributions---in other words, the Bayes error for clustering is high \citep{vankadara2021recovery}. However, as we show in this work, bandit feedback changes the landscape of this problem dramatically, making it possible to resolve non-identifiable clusters when repeated sampling from a distribution is allowed. This opens an entire avenue of practical application of nonparametric distributional clustering in crowdsourcing and other domains where repeated sampling is feasible. 

\paragraph{Potential optimality of \kabc: Comparison with~\citet{thuot2024active}.}
Once non-identifiability is resolved, the natural question is the minimum number of samples needed to recover the clusters; whether \kabc\ is optimal or if the bound in Theorem \ref{theorem:firstAdaptive} could be improved.
We do not provide a proof of the optimality of the bound on the total budget in Eqn. \eqref{eq:bound_budget}, but we discuss why the rates may be optimal. For the case where the arms follow Gaussian distributions \(\nu_i=\mathcal{N}(\mu_i,\sigma_i^2 I_d)\) on \(\mathbb{R}^d\), \citet{thuot2024active} presents the \acb\ algorithm, which has a rate-optimal bound that scales with $\max\limits_{\mu_i\neq\mu_j}\frac{\sigma_i^2\vee\sigma_j^2}{\|\mu_i-\mu_j\|_2^2}\times \left(N\log(\frac{N}{\delta})\vee\sqrt{dKN\log(\frac{N}{\delta})} \right)$.
Note that the Gaussian case, with mean separation, can be handled by \kabc\ with a linear kernel $g(x,x')=x^\top x'$, but the caveat is that the linear kernel is not bounded or translation-invariant (a characteristic kernel is not required if the setting is restricted to Gaussian distributions). 
One could potentially adapt the result by (i) restricting the space to the ball of radius $\sqrt{\bar{g}}$ in $\mathbb{R}^d$, and (ii) adapting Theorem \ref{theorem:firstAdaptive} to use variance-aware bounds for non-translation invariant kernels \citep[Appendix A]{wolfer2025variance}. 
Adapting the relevant modifications, which would also slightly increase the threshold \eqref{def:threshold}, we believe that the modified upper bound on the sampling budget \eqref{eq:bound_budget} scales as $\max\limits_{\mu_i\neq\mu_j}\left(\frac{\sigma_i^2\vee\sigma_j^2}{\|\mu_i-\mu_j\|_2^2}\vee \frac{\sqrt{\bar{g}}}{\|\mu_i-\mu_j\|_2} \right)\times N\log(\frac{N}{\delta})$, which matches the optimal rates of \citet{thuot2024active}, at least in the regime where the clusters have comparable sizes, $\|\mu_i-\mu_j\|\leq \sigma^2_j/\sqrt{\bar g}$ and the dimension is moderate.

\paragraph{Potential optimality of \kabc: Adaptive vs non-adaptive.}
We further comment on the optimality of the adaptive \kabc\ procedure by discussing the optimality of its non-adaptive variant.
Note that the adaptivity of \kabc\ to the unknown signal-to-noise ratio \(s_*\) is achieved by leveraging classical techniques from bandit theory. Notably, \kabc\ achieves a guarantee that scales as \(s_*^{-2}\), without any prior knowledge of the environment, except for the number of clusters. 

One can also define a non-adaptive variant, which takes \(s_*\) as an input parameter and runs the subroutine \texttt{CLUSTER} (Algorithm~\ref{alg:cluster}) with a confidence parameter \(\delta\) and a per-arm sampling budget $n_* \;=\; \frac{128}{s_*^{2}} \log\bigl(\frac{8(N^2-N)}{\delta}\bigr)$.
Lemma~\ref{lemma:CLUSTER} guarantees that the output of \(\texttt{CLUSTER}(n_*,\delta)\) is correct with a probability of at least \(1-\delta\), and the total budget of this non-adaptive procedure is 
$O\left(\frac{N}{s_*^{2}} \log\bigl(\frac{N}{\delta}\bigr)\right)$.
Using the MMD-based two-sample tests of \citet{wolfer2025variance}, this is the state-of-the-art sampling budget that ensures correct clustering of all arms with a global error probability at most \(\delta\). Finally, note that the price of adaptivity is only an additional doubly logarithmic term $O\big(\frac{N}{s_*^{2}}\log\log(s_*^{-2})\big)$, which is negligible compared to the leading term if $\delta$ is small.

\paragraph{Extensions: Unknown number of clusters, grouping close distributions or finding hierarchy.}
In this work, we assume that the number of clusters $K$ is known to the learner and that each cluster contains arms with identical distributions. One can note that $K$ is only used as a stopping criterion and could serve as an upper bound on the desired number of clusters, which could also be set appropriately to group distributions with small MMD in the same cluster.
On a theoretical front, the stopping criterion allows \kabc\ to remain adaptive to the unknown signal-to-noise ratio \(s_*\). As an alternative, if a lower bound \(s_* \geqslant s_0\) is available, then, even without knowing \(K\), a single call \texttt{CLUSTER}\((n_0,\delta)\) with a per-arm sampling budget \(n_0 = \Omega\big( s_0^{-2}\log(\frac{N}{\delta})\big)\) yields a correct partition with high probability $1-\delta$. Mathematically, when neither \(K\) nor \(s_*\) is known, the problem is ill-posed since no algorithm can distinguish between two arm distributions that are arbitrarily close in MMD with finitely many samples.  
An extension of the $K$-clustering problem is to obtain a hierarchical clustering of the distributions, where only a few samples may be needed to obtain coarse-grained clusters, but more rounds may be needed to recover the fine-grained clusters. The study of hierarchical clustering with bandit feedback opens various statistical and algorithmic questions, which we leave for future work.

\section*{AI Disclosure}
We used Claude Opus/Sonnet to assist with designing the experiments and to polish the writing of the introduction. The implementations were done entirely by Claude Sonnet, and the important parts were verified by one author.
On a lighter note, the following conservative guarantees hold for the LLM-based implementation: (i) the reported results are not hallucinated, as `some code' provably ran on the computer, and (ii) with high probability, the generated codes are more reliable than those written by a uniformly sampled toddler who has not yet started using an AI assistant.

% Acknowledgments---Will not appear in anonymized version
\acks{
V. Thuot and S. Vogt contributed equally. The work of V. Thuot was done during his doctoral research at INRAE Montpellier. S. Vogt contributed to this work when he was a Master student at Technical University of Munich.
The authors acknowledge funding support from the German Research Foundation (DFG) and by the ANR through the DFG-ANR PRCI ``ASCAI" (GH 257/3-1 ;  ANR-21-CE23-0035) and the DFG research grant GH 257/4-1. 
}

% Replace yourbibfile with the name of your BibTeX database and uncomment:

\bibliography{bibliography}

\begin{thebibliography}{30}
\providecommand{\natexlab}[1]{#1}
\providecommand{\url}[1]{\texttt{#1}}
\expandafter\ifx\csname urlstyle\endcsname\relax
  \providecommand{\doi}[1]{doi: #1}\else
  \providecommand{\doi}{doi: \begingroup \urlstyle{rm}\Url}\fi

\bibitem[Aragam et~al.(2020)Aragam, Dan, Xing, and Ravikumar]{aragam2020identifiability}
Bryon Aragam, Chen Dan, Eric~P. Xing, and Pradeep Ravikumar.
\newblock {Identifiability of nonparametric mixture models and Bayes optimal clustering}.
\newblock \emph{The Annals of Statistics}, 48\penalty0 (4):\penalty0 2277 -- 2302, 2020.

\bibitem[Ba{\'{\i}}llo et~al.(2025)Ba{\'{\i}}llo, Berrendero, and S{\'{a}}nchez{-}Signorini]{baillo2025kernel}
Amparo Ba{\'{\i}}llo, Jos{\'{e}}~R. Berrendero, and Mart{\'{\i}}n S{\'{a}}nchez{-}Signorini.
\newblock Kernel k-means clustering of distributional data.
\newblock \emph{arXiv preprint}, arXiv:2509.18037, 2025.

\bibitem[Berlinet and Thomas-Agnan(2011)]{berlinet2011reproducing}
Alain Berlinet and Christine Thomas-Agnan.
\newblock \emph{Reproducing kernel Hilbert spaces in probability and statistics}.
\newblock Springer Science \& Business Media, 2011.

\bibitem[Chandran et~al.(2025)Chandran, Reddy, and Bhashyam]{chandran2025online}
G.~Dhinesh Chandran, Kota~Srinivas Reddy, and Srikrishna Bhashyam.
\newblock Online clustering with bandit information.
\newblock In \emph{{IEEE} International Symposium on Information Theory, {ISIT} 2025, Ann Arbor, MI, USA, June 22-27, 2025}, pages 1--6. IEEE, {IEEE}, 2025.

\bibitem[Couillet and Benaych-Georges(2016)]{CouilletBG16}
Romain Couillet and Florent Benaych-Georges.
\newblock {Kernel spectral clustering of large dimensional data}.
\newblock \emph{Electronic Journal of Statistics}, 10\penalty0 (1):\penalty0 1393 -- 1454, 2016.

\bibitem[Dhillon et~al.(2004)Dhillon, Guan, and Kulis]{dhillon2004kernel}
Inderjit~S. Dhillon, Yuqiang Guan, and Brian Kulis.
\newblock Kernel k-means: spectral clustering and normalized cuts.
\newblock In Won Kim, Ron Kohavi, Johannes Gehrke, and William DuMouchel, editors, \emph{Proceedings of the Tenth {ACM} {SIGKDD} International Conference on Knowledge Discovery and Data Mining, Seattle, Washington, USA, August 22-25, 2004}, pages 551--556. {ACM}, 2004.

\bibitem[Fleissner et~al.(2026)Fleissner, Zarvandi, and Ghoshdastidar]{fleissner26a}
Maximilian Fleissner, Maedeh Zarvandi, and Debarghya Ghoshdastidar.
\newblock Explainable clustering of mixture models.
\newblock In Emilija Perković and Daniel Malinsky, editors, \emph{Proceedings of the 42nd Conference on Uncertainty in Artificial Intelligence}, volume 337 of \emph{Proceedings of Machine Learning Research}, pages 1585--1603. PMLR, 17--21 Aug 2026.

\bibitem[Garivier and Kaufmann(2016)]{garivier2016optimal}
Aur{\'{e}}lien Garivier and Emilie Kaufmann.
\newblock Optimal best arm identification with fixed confidence.
\newblock In Vitaly Feldman, Alexander Rakhlin, and Ohad Shamir, editors, \emph{Proceedings of the 29th Conference on Learning Theory, {COLT} 2016, New York, USA, June 23-26, 2016}, volume~49 of \emph{{JMLR} Workshop and Conference Proceedings}, pages 998--1027. JMLR.org, 2016.

\bibitem[Graf et~al.(2025)Graf, Thuot, and Verzelen]{graf2025clustering}
Maximilian Graf, Victor Thuot, and Nicolas Verzelen.
\newblock Clustering items through bandit feedback: Finding the right feature out of many.
\newblock In Aarti Singh, Maryam Fazel, Daniel Hsu, Simon Lacoste{-}Julien, Felix Berkenkamp, Tegan Maharaj, Kiri Wagstaff, and Jerry Zhu, editors, \emph{Forty-second International Conference on Machine Learning, {ICML} 2025, Vancouver, BC, Canada, July 13-19, 2025}, volume 267 of \emph{Proceedings of Machine Learning Research}. {PMLR} / OpenReview.net, 2025.

\bibitem[Gretton et~al.(2007{\natexlab{a}})Gretton, Borgwardt, Rasch, Sch{\"{o}}lkopf, and Smola]{gretton2007akernel}
Arthur Gretton, Karsten~M. Borgwardt, Malte~J. Rasch, Bernhard Sch{\"{o}}lkopf, and Alexander~J. Smola.
\newblock A kernel approach to comparing distributions.
\newblock In \emph{Proceedings of the Twenty-Second {AAAI} Conference on Artificial Intelligence, July 22-26, 2007, Vancouver, British Columbia, Canada}, pages 1637--1641. {AAAI} Press, 2007{\natexlab{a}}.

\bibitem[Gretton et~al.(2007{\natexlab{b}})Gretton, Fukumizu, Teo, Song, Sch{\"{o}}lkopf, and Smola]{gretton2007bkernel}
Arthur Gretton, Kenji Fukumizu, Choon~Hui Teo, Le~Song, Bernhard Sch{\"{o}}lkopf, and Alexander~J. Smola.
\newblock A kernel statistical test of independence.
\newblock In John~C. Platt, Daphne Koller, Yoram Singer, and Sam~T. Roweis, editors, \emph{Advances in Neural Information Processing Systems 20, Proceedings of the Twenty-First Annual Conference on Neural Information Processing Systems, Vancouver, British Columbia, Canada, December 3-6, 2007}, pages 585--592. Curran Associates, Inc., 2007{\natexlab{b}}.

\bibitem[Gretton et~al.(2012)Gretton, Borgwardt, Rasch, Sch{\"o}lkopf, and Smola]{gretton2012kernel}
Arthur Gretton, Karsten~M Borgwardt, Malte~J Rasch, Bernhard Sch{\"o}lkopf, and Alexander Smola.
\newblock A kernel two-sample test.
\newblock \emph{J. Mach. Learn. Res.}, 13\penalty0 (1):\penalty0 723--773, 2012.

\bibitem[Harper and Konstan(2016)]{HarperK16}
F.~Maxwell Harper and Joseph~A. Konstan.
\newblock The movielens datasets: History and context.
\newblock \emph{{ACM} Trans. Interact. Intell. Syst.}, 5\penalty0 (4):\penalty0 19:1--19:19, 2016.

\bibitem[Jain(2010)]{jain2010data}
Anil~K. Jain.
\newblock Data clustering: 50 years beyond {K}-means.
\newblock \emph{Pattern Recognition Letters}, 31\penalty0 (8):\penalty0 651--666, 2010.

\bibitem[Lattimore and Szepesv{\'a}ri(2020)]{lattimore2020bandit}
Tor Lattimore and Csaba Szepesv{\'a}ri.
\newblock \emph{Bandit Algorithms}.
\newblock Cambridge University Press, 2020.

\bibitem[Muandet et~al.(2017)Muandet, Fukumizu, Sriperumbudur, and Sch{\"{o}}lkopf]{muandet2017kernel}
Krikamol Muandet, Kenji Fukumizu, Bharath~K. Sriperumbudur, and Bernhard Sch{\"{o}}lkopf.
\newblock Kernel mean embedding of distributions: {A} review and beyond.
\newblock \emph{Found. Trends Mach. Learn.}, 10\penalty0 (1-2):\penalty0 1--141, 2017.

\bibitem[Ng et~al.(2001)Ng, Jordan, and Weiss]{ng2001spectral}
Andrew Ng, Michael Jordan, and Yair Weiss.
\newblock On spectral clustering: Analysis and an algorithm.
\newblock \emph{Advances in neural information processing systems}, 14, 2001.

\bibitem[Smola et~al.(2007)Smola, Gretton, Song, and Sch{\"{o}}lkopf]{smola2007hilbert}
Alexander~J. Smola, Arthur Gretton, Le~Song, and Bernhard Sch{\"{o}}lkopf.
\newblock A hilbert space embedding for distributions.
\newblock In Marcus Hutter, Rocco~A. Servedio, and Eiji Takimoto, editors, \emph{Algorithmic Learning Theory, 18th International Conference, {ALT} 2007, Sendai, Japan, October 1-4, 2007, Proceedings}, volume 4754 of \emph{Lecture Notes in Computer Science}, pages 13--31. Springer, 2007.

\bibitem[Sriperumbudur et~al.(2010)Sriperumbudur, Gretton, Fukumizu, Sch{\"{o}}lkopf, and Lanckriet]{sriperumbudur2010hilbert}
Bharath~K. Sriperumbudur, Arthur Gretton, Kenji Fukumizu, Bernhard Sch{\"{o}}lkopf, and Gert R.~G. Lanckriet.
\newblock Hilbert space embeddings and metrics on probability measures.
\newblock \emph{J. Mach. Learn. Res.}, 11:\penalty0 1517--1561, 2010.

\bibitem[Sriperumbudur et~al.(2011)Sriperumbudur, Fukumizu, and Lanckriet]{sriperumbudur2011universality}
Bharath~K. Sriperumbudur, Kenji Fukumizu, and Gert R.~G. Lanckriet.
\newblock Universality, characteristic kernels and {RKHS} embedding of measures.
\newblock \emph{J. Mach. Learn. Res.}, 12:\penalty0 2389--2410, 2011.

\bibitem[Thuot et~al.(2025)Thuot, Carpentier, Giraud, and Verzelen]{thuot2024active}
Victor Thuot, Alexandra Carpentier, Christophe Giraud, and Nicolas Verzelen.
\newblock Clustering with bandit feedback: breaking down the computation/information gap.
\newblock In Gautam Kamath and Po{-}Ling Loh, editors, \emph{International Conference on Algorithmic Learning Theory, 24-27 February 2025, Politecnico di Milano, Milan, Italy}, volume 272 of \emph{Proceedings of Machine Learning Research}, pages 1221--1284. {PMLR}, 2025.

\bibitem[Tolstikhin et~al.(2016)Tolstikhin, Sriperumbudur, and Sch{\"{o}}lkopf]{tolstikhin2016minimax}
Ilya~O. Tolstikhin, Bharath~K. Sriperumbudur, and Bernhard Sch{\"{o}}lkopf.
\newblock Minimax estimation of maximum mean discrepancy with radial kernels.
\newblock In Daniel~D. Lee, Masashi Sugiyama, Ulrike von Luxburg, Isabelle Guyon, and Roman Garnett, editors, \emph{Advances in Neural Information Processing Systems 29: Annual Conference on Neural Information Processing Systems 2016, December 5-10, 2016, Barcelona, Spain}, pages 1930--1938, 2016.

\bibitem[Tolstikhin et~al.(2017)Tolstikhin, Sriperumbudur, and Muandet]{tolstikhin2017minimax}
Ilya~O. Tolstikhin, Bharath~K. Sriperumbudur, and Krikamol Muandet.
\newblock Minimax estimation of kernel mean embeddings.
\newblock \emph{J. Mach. Learn. Res.}, 18:\penalty0 86:1--86:47, 2017.

\bibitem[Vankadara et~al.(2021)Vankadara, Bordt, von Luxburg, and Ghoshdastidar]{vankadara2021recovery}
Leena~C. Vankadara, Sebastian Bordt, Ulrike von Luxburg, and Debarghya Ghoshdastidar.
\newblock Recovery guarantees for kernel-based clustering under non-parametric mixture models.
\newblock In Arindam Banerjee and Kenji Fukumizu, editors, \emph{The 24th International Conference on Artificial Intelligence and Statistics, {AISTATS} 2021, April 13-15, 2021, Virtual Event}, volume 130 of \emph{Proceedings of Machine Learning Research}, pages 3817--3825. {PMLR}, 2021.

\bibitem[Vankadara and Ghoshdastidar(2020)]{VankadaraG20}
Leena~Chennuru Vankadara and Debarghya Ghoshdastidar.
\newblock On the optimality of kernels for high-dimensional clustering.
\newblock In Silvia Chiappa and Roberto Calandra, editors, \emph{The 23rd International Conference on Artificial Intelligence and Statistics, {AISTATS} 2020, 26-28 August 2020, Online [Palermo, Sicily, Italy]}, volume 108 of \emph{Proceedings of Machine Learning Research}, pages 2185--2195. {PMLR}, 2020.

\bibitem[von Luxburg(2007)]{vonluxburg2007tutorial}
Ulrike von Luxburg.
\newblock A tutorial on spectral clustering.
\newblock \emph{Statistics and Computing}, 17\penalty0 (4):\penalty0 395--416, 2007.

\bibitem[Wolfer and Alquier(2025)]{wolfer2025variance}
Geoffrey Wolfer and Pierre Alquier.
\newblock Variance-aware estimation of kernel mean embedding.
\newblock \emph{J. Mach. Learn. Res.}, 26:\penalty0 57:1--57:48, 2025.

\bibitem[Yan and Sarkar(2016)]{YanS16}
Bowei Yan and Purnamrita Sarkar.
\newblock On robustness of kernel clustering.
\newblock In Daniel~D. Lee, Masashi Sugiyama, Ulrike von Luxburg, Isabelle Guyon, and Roman Garnett, editors, \emph{Advances in Neural Information Processing Systems 29: Annual Conference on Neural Information Processing Systems 2016, December 5-10, 2016, Barcelona, Spain}, pages 3090--3098, 2016.

\bibitem[Yang et~al.(2024)Yang, Zhong, and Tan]{yang2024optimal}
Junwen Yang, Zixin Zhong, and Vincent Y.~F. Tan.
\newblock Optimal clustering with bandit feedback.
\newblock \emph{J. Mach. Learn. Res.}, 25:\penalty0 186:1--186:54, 2024.

\bibitem[Yavas et~al.(2025)Yavas, Huang, Tan, and Scarlett]{yavas2025general}
Recep~Can Yavas, Yuqi Huang, Vincent Y.~F. Tan, and Jonathan Scarlett.
\newblock A general framework for clustering and distribution matching with bandit feedback.
\newblock \emph{{IEEE} Trans. Inf. Theory}, 71\penalty0 (3):\penalty0 2116--2139, 2025.

\end{thebibliography}

\appendix

% \crefalias{section}{appendix} % uncomment if you are using cleveref

\section{Proofs of Theorem~\ref{theorem:firstAdaptive}}\label{appendix:proof_UB}

In this appendix, we prove the correctness of \kabc\ (Algorithm~\ref{algo:KABC}) and derive a high-probability upper bound on its budget. 

\subsection{Proof of Theorem~\ref{theorem:firstAdaptive}}

Recall from Section~\ref{sec:algorithm} that \kabc\ repeatedly calls the subroutine \cluster$(n_k,\delta_k)$ with increasing per-arm budgets $n_k$ and decreasing confidence levels $\delta_k$.
At iteration $k$, \cluster\ constructs a graph by performing variance-aware kernel two-sample tests between all pairs of arms and returns the connected components as a clustering $\mathcal{C}_k$. For any arm $i \in [N]$, let $C_k(i)$ denote the unique cluster in $\mathcal{C}_k$ containing $i$.

The analysis of \kabc\ therefore reduces to understanding, for each fixed $k$, how \cluster\ behaves in terms of type~I and type~II errors under the thresholds $\mathcal{B}^{i,j}(n_k,\delta_k)$ introduced in Equation~\eqref{def:threshold}.

Intuitively, we need that $\cluster(n_k,\delta_k)$ either returns fewer than $K$ clusters or identifies the correct $K$ clusters. The condition $\|\hat{\mu}_i - \hat{\mu}_j\| \leq \mathcal{B}^{i,j}(n_k,\delta_k)$ ensures that arms $i$ and $j$ with $\mu_i = \mu_j$ are clustered together with high probability. Moreover, there exists an iteration $k$ such that arms with $\mu_i \neq \mu_j$ are separated (i.e., $\|\hat{\mu}_i - \hat{\mu}_j\| > \mathcal{B}^{i,j}(n_k,\delta_k)$) with high probability, yielding exactly $K$ clusters. These key properties of $\cluster$ are formalized in Lemma~\ref{lemma:CLUSTER}, whose proof---relying solely on concentration inequalities from Appendix~\ref{appendix:concentration}---is deferred to Appendix~\ref{appendix:lemma_cluster}.

\begin{lemma}\label{lemma:CLUSTER}
For a fixed iteration $k\geq 1$, consider the clustering $\mathcal{C}_k$ returned by \cluster$(n_k,\delta_k)$, and let $C_k(i)$ denote the cluster containing arm $i$.

Define the event \emph{type~I error event} $\mathcal{E}_{1,k}$ under which arms with identical KMEs are assigned to different clusters in $\mathcal{C}_k$:
\begin{equation}\label{eq:type_1_error}
    \mathcal{E}_{1,k}=\bigcup_{\substack{i < j \in [N]\\\mu_i = \mu_j}}\Big\{C_k(i) \neq C_k(j)\Big\} \enspace. 
\end{equation}
 Then,
\[ \forall k\geqslant 1, \: \mathbb{P}(\mathcal{E}_{1,k})\leqslant \tfrac{\delta_k}{2} \enspace.\]

Define the \emph{type~II error event} $\mathcal{E}_{2,k}$ under which arms with distinct KMEs are assigned to the same cluster in $\mathcal{C}_k$:
\begin{equation}\label{eq:type_2_error}
    \mathcal{E}_{2,k}=\bigcup_{\substack{i < j \in [N]\\\mu_i \ne \mu_j}}\Big\{C_k(i) = C_k(j)\Big\} \enspace. 
\end{equation}

Let $s_*=s_*(\nu)$ be given in Equation~\eqref{def:SNR}  by
$s^2_*(\nu):= \min_{\substack{i < j \in [N]\\\mu_i \neq \mu_j}} \frac{{\|\mu_i-\mu_j\|}^2}{\mathcal{V}^*_i \vee \mathcal{V}^*_j } \wedge \frac{2{\|\mu_i-\mu_j\|
    }}{\sqrt{\bar{g}}}$.
 
 For any iteration $\forall k\geqslant 1$, such that $n_k\geqslant 128\frac{1}{s_*^2}\log\left(\frac{8(N^2-N)}{\delta_k}\right)$, then 
    \[ \mathbb{P}(\mathcal{E}_{2,k}) \leqslant \tfrac{\delta_k}{2}\enspace.\]
\end{lemma}

\begin{proof}[Proof of Theorem~\ref{theorem:firstAdaptive}]
   Let $\delta \in (0,1]$ and assume the environment $\nu$ has $N$ arms forming exactly $K$ groups with distinct KMEs. We proceed in three parts: $(i)$ correctness (outputting the true partition with probability at least $1 - \delta$), $(ii)$ almost-sure termination, and $(iii)$ a high-probability bound on the sampling complexity.

 \paragraph{\underline{$(i)$ Correctness of $\kabc(\delta,K)$}}

 Consider the call of $\kabc(\delta,K)$ on the environment $\nu$.  Let $\mathcal{C}$ be the clustering returned by \kabc$(\delta,K)$, obtained from $\cluster(n_k,\delta_k)$ at the stopping iteration $k \geq 1$. Observe that $k<+\infty$ happens almost surely (as we prove in (ii)), and the clustering $\mathcal{C}$ contains exactly $K$ groups. 
 On the event $\mathcal{E}_{1,k}^c$ (Equation~\eqref{eq:type_1_error}), all arms with identical KMEs are clustered together. Since, by definition of the model in Section~\ref{sec:setting}, the true partition $\mathcal{C}^*$ is exactly the partition into equivalence classes of the relation $\mu_i=\mu_j$ and has cardinality $K$, any clustering with no type~I error and exactly $K$ nonempty clusters must coincide with $\mathcal{C}^*$ (up to permutation of the labels).
 Since the stopping iteration $k$ is unknown, we proceed with a union bound, and bound the Type I error with Lemma~\ref{lemma:CLUSTER}:
\begin{align*}
        \P{\left\{\mathcal{C}\text{ is incorrect}\right\}}
        =& \P{\event{\exists k \geq 1 \, \colon \mathcal{C}_k\text{ is incorrect} \text{ and} \left|\mathcal{C}_k\right| = K}} \\
        \leq{}& \sum_{k=1}^{\infty}\P{\mathcal{E}_{1,k}} \\
        \leq{}& \sum_{k=1}^{\infty}\frac{\delta_k}{2} \leq \delta \enspace, 
    \end{align*}
    where the final bound uses $\sum_{k=1}^\infty 1/k^2 = \pi^2/6$ and $\delta_k = \delta/(4k^2)$. This establishes the correctness of Algorithm~\ref{algo:KABC}.

 \paragraph{\underline{$(ii)$ Almost-sure termination of \kabc$(\delta,K)$}}

 On $\mathcal{E}_{1,k}^c \cap \mathcal{E}_{2,k}^c$, we have that $\mathcal{C}_k$ clusters the arms exactly according to their KMEs, and $\mathcal{C}_k$ equals the true partition. In particular $|\mathcal{C}_k| = K$, and the algorithm terminates.  
% In other words, once $n_k$ is large enough , the graph built by \cluster exactly matches the true equivalence relation $\mu_i=\mu_j$, and \kabc\ stops at that iteration.

Then, 
\begin{align*}
        & \P{\kabc(\delta,K)\text{ never terminates}} \leqslant \inf_{k\to\infty} \mathbb{P}(\mathcal{E}_{1,k} \cup \mathcal{E}_{2,k}) \enspace.
    \end{align*}
Moreover, for $k$ large enough, it holds that $n_k \geq 128 s_*^{-2} \log\big(8(N^2-N)/\delta_k\big)$. Lemma~\ref{lemma:CLUSTER} then implies that for $k$ large enough, $\mathbb{P}(\mathcal{E}_{1,k} \cup \mathcal{E}_{2,k})\leqslant \delta_k$, and  $\lim_{k\to\infty}\mathbb{P}(\mathcal{E}_{1,k} \cup \mathcal{E}_{2,k})=0$.
Finally, 
\begin{align*}
        & \P{\kabc(\delta,K)\text{ never  terminates}} \leqslant \inf_{k\to\infty} \mathbb{P}(\mathcal{E}_{1,k} \cup \mathcal{E}_{2,k})=0\enspace,
    \end{align*}
and the algorithm terminates almost surely. 
    
\paragraph{\underline{$(iii)$ Budget of $\kabc(\delta,K)$}}

Denote
\begin{equation}\label{def:k_*}
    k_*=\left\lceil \log_2\left(\frac{128}{s_*^2}\right)\right\rceil \vee 1 \enspace,
\end{equation}
so that $n_k \geq 128 s_*^{-2} \log\big(8(N^2-N)/\delta_k\big)$ for all $k \geq k_*$. Lemma~\ref{lemma:CLUSTER} then implies $\mathbb{P}(\mathcal{E}_{1,k} \cup \mathcal{E}_{2,k}) \leq \delta_k \leq \delta_{k_*}$ for $k \geq k_*$.

On $\mathcal{E}^c_{1,k_*}\cap\mathcal{E}^c_{2,k_*}$, the algorithm terminates at or before iteration $k_*$.  The per-iteration budget is $\tau_k = N \cdot n_k$, so the total budget $\tau$ satisfies
 \begin{align*}
        & \P{\tau \leq \sum_{k=1}^{k_*}\tau_k} \geq \P{\mathcal{E}^c_{1,k_*}\cap\mathcal{E}^c_{2,k_*}}  \geq 1 - \delta_{k_*}\geqslant 1-\delta \enspace.
    \end{align*}
    
    Now,
    \begin{align*}
         \sum_{k=1}^{k_*}\tau_k = &N \sum_{k=1}^{k_*} n_k = N \sum_{k=1}^{k_*} \left\lceil 2^k \log\frac{8(N^2 - N)}{\delta_k} \right\rceil \\
        \leq{}& 2N\sum_{k=1}^{k_*} 2^{k}\cdot \log\left(\frac{8(N^2 - N)}{\delta_{k_*}} \right) \\
        \leq{}& 4N\cdot 2^{k_*}\cdot \log\left(\frac{8(N^2 - N)}{\delta_{k_*}} \right) \\
        \leq{}&8N \cdot\left(\frac{128}{s_*^2}\vee 1\right)\cdot \left(\log\left(\frac{32(N^2 - N)k_*^2}{\delta} \right)\right) \enspace.
    \end{align*}
    This provides the high probability bound on the budget. 
\end{proof}

\subsection{Proof of  Lemma~\ref{lemma:CLUSTER}}\label{appendix:lemma_cluster}

 \begin{proof}

We fix any iteration $k$, for which we call $\cluster(n_k,\delta_k)$.  Recall that the algorithm $\cluster$ constructs an undirected graph $G_k=([N],E_k)$, whose vertices are the arms $[N]$, and whose set of edges is $ E_k=\{\{i,j\} : \|\hat\mu_i-\hat\mu_j\|\leqslant \mathcal{B}^{i,j}(n_k,\delta_k)\}$.  Then, the clustering $\mathcal{C}_k$ is defined as the connected components of $G_k$.  

\textbf{Type-I error: splitting arms with identical KME.} 
   
We consider the event $\mathcal{E}_{1,k}$ (Equation~\eqref{eq:type_1_error}), under which 
$\mathcal{C}_k$ assigns two arms with the same KME to different clusters.
Let $i\ne j$ be such that that $C_k(i) \neq C_k(j)$ and with $\mu_i=\mu_j$. The set $E_k$ cannot contain the edge $\left\{i, j\right\}$, because otherwise they would be in the same connected component, and thus in the same cluster. By definition of $E_k$,  as $(i,j)\not\in E_k$, it means that the comparison of $\empdist$ to the decision boundary yielded false. It follows by construction that 
     \begin{equation*}
        \mathcal{E}_{1,k} \subset \bigcup_{\substack{i < j \in [N]\\\mu_i = \mu_j}}\event{\empdist > \mathcal{B}^{i,j}(n_k,\delta_k)},
    \end{equation*}
where 
  \begin{equation*}
        \mathcal{B}^{i,j}(n_k,\delta_k) = \left(\sqrt{\hat{\mathcal{V}}_i} + \sqrt{\hat{\mathcal{V}}_j}\right)\sqrt{2\frac{\log\frac{8(N^2 - N)}{\delta_k}}{n_k}} + \frac{32}{3}\sqrt{\tilde{g}}\frac{\log\frac{8(N^2 - N)}{\delta_k}}{n_k} \enspace.
    \end{equation*}
    
Now, we bound the probability $\mathbb{P}(\mathcal{E}_{1,k})$, with a union bound, together with the concentration inequality from Lemma~\ref{lemma:varaw}  with $\delta' = \frac{\delta_k}{N^2 - N}$,
    \begin{align*}
       \P{\mathcal{E}_{1,k}} 
& \leq \sum_{\substack{i < j \in [N]\\\mu_i = \mu_j}} \P{\event{\|\hat\mu_i-\hat\mu_j\| > \mathcal{B}^{i,j}(n_k,\delta_k)}} \\ 
& \leq \sum_{\substack{i < j \in [N]\\\mu_i = \mu_j}} \frac{\delta_k}{N^2 - N}  \leqslant \tfrac{\delta_k}{2} \enspace. 
    \end{align*}

\textbf{Type-II error: merging arms with different KMEs.} 

We now consider  the event $\mathcal{E}_{2,k}$ (Equation~\eqref{eq:type_2_error}), under which $\mathcal{C}_k$ assigns two arms with different KMEs to the same cluster. For simplicity, we write
\begin{equation*}
    \delta_{N,k}=\frac{\delta_k}{(N^2-N)}=\frac{\delta}{4\cdot(N^2-N)\cdot k^2} \enspace.
\end{equation*}

    Assume that $\mu_i\ne \mu_j$, while $C_k(i) = C_k(j)$. By construction of $\mathcal C_k$, the condition  $C_k(i) = C_k(j)$ implies that there exists a path in the graph $G_k=([N],E_k)$ between $i$ and $j$. Since there is a path between two arms with different means, somewhere on that path must be an edge connecting two arms with different means. Thus, it holds that
    \begin{equation}\label{eq:inclusion_error_2}
       \mathcal{E}_{2,k} \subset \bigcup_{\substack{i<j \in [N]\\\mu_i \ne \mu_j}}\event{\empdist \leqslant \mathcal{B}^{i,j}(n_k,\delta_k)}\enspace.
    \end{equation}

   Let $i,j\in[N]$ be two arms such that $\mu_i\ne \mu_j$.
   Now we need to control the probability of the event $A_{i,j} \coloneqq \event{\empdist \leq \mathcal{B}^{i,j}(n_k,\delta_k)}$.

Lemma~\ref{lemma:tolstikhin} provides us that with probability at least $1 - \frac{\delta_{N,k}}{2}$
    \begin{equation*}
        \empdist \geq \norm{\mu_j - \mu_i} - \left(\sqrt{\mathcal{V}_i^*} + \sqrt{\mathcal{V}_j^*}\right)\sqrt{2\frac{\log\frac{8}{\delta_{N,k}}}{n_k}} - \frac{8}{3}\sqrt{\bar{g}}\frac{\log\frac{8}{\delta_{N,k}}}{n_k} \enspace.
    \end{equation*}
It remains to show that, under Assumption~\eqref{eq:assump_n_k} on $n_k$, the lower bound from Lemma~\ref{lemma:tolstikhin} exceeds the threshold $\mathcal{B}^{i,j}(n_k,\delta_k)$
with high probability.

    Since $\mathcal{B}^{i,j}(n_k,\delta_k)$ contains the empirical variance, we first need to bound the empirical variance by the true variance. For any arm $i$, we define the event
    \[
    B_i \coloneqq \event{\sqrt{\empvar{i}} \leq \sqrt{\mathcal{V}_i^*} + 2 \sqrt{\frac{2 \tilde{g} \log\frac{8}{\delta_{N,k}}}{n_k}}}\enspace,
    \]
whose probability is controlled by Lemma~\ref{lemma:empVarBound} as
    \begin{equation}\label{eq:controle_variance}
        \P{B_i^c} \leq \frac{\delta_{N,k}}{8} \enspace.
    \end{equation}
    With a union bound on Equation~\eqref{eq:inclusion_error_2}, we have that 
\begin{align}
    \mathbb{P}(\mathcal{E}_{2,k}) & \leq  \sum_{\substack{i<j \in [N]\\\mu_i \ne \mu_j}}\P{A_{i,j}} \nonumber\\
    & \leq  \sum_{\substack{i<j \in [N]\\\mu_i \ne \mu_j}}\left(\P{A_{i,j} \cap B_i \cap B_j}+\P{B_i^c}+\P{B_j^c}\right) \nonumber \\ 
  & \leq  \sum_{\substack{i<j \in [N]\\\mu_i \ne \mu_j}}\left(\P{A_{i,j} \cap B_i \cap B_j}+\frac{\delta_{N,k}}{4}\right) \label{eq:controle_error_2_a}\enspace,
\end{align}
where the final inequality follows from Equation~\eqref{eq:controle_variance}.  Then, it remains  to bound the probability of the events $A_{i,j} \cap B_i \cap B_j$.

  Under $A_{i,j} \cap B_i \cap B_j$, it holds that
    \begin{align}
        & \empdist \leq \left(\sqrt{\hat{\mathcal{V}}_i} + \sqrt{\hat{\mathcal{V}}_j}\right)\sqrt{2\frac{\log\frac{8}{\delta_{N,k}}}{n_k}} + \frac{32}{3}\sqrt{\tilde{g}}\frac{\log\frac{8}{\delta_{N,k}}}{n_k} \notag \\
        \leq{}& \left(\sqrt{\mathcal{V}_i^*} + \sqrt{\mathcal{V}_j^*} + 4 \sqrt{\frac{2 \tilde{g} \log\frac{8}{\delta_{N,k}}}{n_k}} \right) \sqrt{2\frac{\log\frac{8}{\delta_{N,k}}}{n_k}} + \frac{32}{3}\sqrt{\tilde{g}}\frac{\log\frac{8}{\delta_{N,k}}}{n_k} \notag \\
        =& \left(\sqrt{\mathcal{V}_i^*} + \sqrt{\mathcal{V}_j^*}\right) \sqrt{2\frac{\log\frac{8}{\delta_{N,k}}}{n_k}} + \frac{56}{3}\sqrt{\tilde{g}}\frac{\log\frac{8}{\delta_{N,k}}}{n_k} \enspace.\label{eq:ABiBj}
    \end{align}
    
    Now, assume that 
    \begin{equation}\label{eq:assump_n_k}
        n_k \geq \max_{\substack{i \neq j \in [N]\\\mu_i \neq \mu_j}} \left\{128\frac{\mathcal{V}^*_i \vee \mathcal{V}^*_j}{{\|\mu_i-\mu_j\|}^2} \bigvee \frac{112 \sqrt{\tilde{g}} + 16\sqrt{\bar{g}}}{3{\|\mu_i-\mu_j\|}} \right\}\log\frac{8}{\delta_{N,k}} \enspace.
    \end{equation}
    In particular, as $\tilde{g}\leqslant 2\bar{g}$, this will hold if $n_k\geqslant 128\frac{1}{s_*^2}\log\frac{8}{\delta_{N,k}}$. 
    
   Now,  we derive from direct computation that Assumption~\eqref{eq:assump_n_k} on $n_k$ implies
       \begin{align}
       & \left(\sqrt{\mathcal{V}_i^*} + \sqrt{\mathcal{V}_j^*}\right) \sqrt{2\frac{\log\frac{8}{\delta_{N,k}}}{n_k}} + \frac{56}{3}\sqrt{\tilde{g}}\frac{\log\frac{8}{\delta_{N,k}}}{n_k}  \nonumber
    \\ \leqslant  & \norm{\mu_j - \mu_i} - \left(\sqrt{\mathcal{V}_i^*} + \sqrt{\mathcal{V}_j^*}\right)\sqrt{2\frac{\log\frac{8}{\delta_{N,k}}}{n_k}} - \frac{8}{3}\sqrt{\bar{g}}\frac{\log\frac{8}{\delta_{N,k}}}{n_k} \enspace.\label{eq:UBmatchLB}
    \end{align}

    Finally, we gather Equations~\eqref{eq:ABiBj}, and~\eqref{eq:UBmatchLB} to control $\P{A_{i,j} \cap B_i \cap B_j}$, 
    \begin{align*}
        & \P{A_{i,j} \cap B_i \cap B_j}\\
        \leq{}& \P{{\empdist \leq \left( \sqrt{\mathcal{V}_i^*} + \sqrt{\mathcal{V}_j^*} \right) \sqrt{2\frac{\log\frac{8}{\delta_{N,k}}}{n_k}} + \frac{56}{3}\sqrt{\tilde{g}}\frac{\log\frac{8}{\delta_{N,k}}}{n_k}}} \\
        \leq{}& \P{{\empdist \leq \norm{\mu_j - \mu_i} - \left(\sqrt{\mathcal{V}_i^*} + \sqrt{\mathcal{V}_j^*}\right)\sqrt{2\frac{\log\frac{8}{\delta_{N,k}}}{n_k}} - \frac{8}{3}\sqrt{\bar{g}}\frac{\log\frac{8}{\delta_{N,k}}}{n_k}}} \\
        \leq{}& \frac{\delta_{N,k}}{2},
    \end{align*}
    where the last inequality results from Lemma~\ref{lemma:tolstikhin} applied with $\delta'=\frac{\delta_{N,k}}{2}$.

Finally, we come back to our previous bound on $\mathbb{P}(\mathcal{E}_{2,k})$ in Equation~\eqref{eq:controle_error_2_a}, to conclude that 

\begin{align}
    \mathbb{P}(\mathcal{E}_{2,k}) 
  & \leq  \sum_{\substack{i<j \in [N]\\\mu_i \ne \mu_j}}\left(\P{A_{i,j} \cap B_i \cap B_j}+\frac{\delta_{N,k}}{4}\right) \nonumber \\
    & \leq  \sum_{\substack{i<j \in [N]\\\mu_i \ne \mu_j}}\left(\frac{\delta_{N,k}}{2}+\frac{\delta_{N,k}}{4}\right) \nonumber \\
    & \leqslant \tfrac{\delta_k}{2}  \nonumber\enspace ,
\end{align}
where the final inequality follows from the definition of $\delta_{N,k}=\frac{\delta_k}{(N^2-N)}$
 \end{proof}

\subsection{Concentration inequalities}\label{appendix:concentration}

The design of the algorithm relies mostly on concentration inequalities for empirical KME, and empirical variances. For completeness, we recall in this subsection several inequalities whose proofs can be found in \citep{wolfer2025variance}. 

\begin{lemma}[Variance-aware empirical bound \citep{wolfer2025variance}]\label{lemma:varianceAwareEmpBound}
    Let $g$ be a continuous, positive definite, characteristic, translation invariant, bounded kernel.
    Assume arm $i \in [N]$ was sampled $n \in \Nnums$ times to calculate the empirical KME $\hat{\mu}_i$ and the empirical variance $\hat{\mathcal{V}}_i$.
    Then it holds that
    \begin{equation*}
        \forall \delta' \in (0, 1] \colon \P{\event{\emptoreal{i} \leq \sqrt{2 \empvar{i} \frac{\log\frac{4}{\delta'}}{n}} + \frac{16}{3}\sqrt{\tilde{g}}\frac{\log\frac{4}{\delta'}}{n}}} \geq 1 - \delta'.
    \end{equation*}
\end{lemma}

For our application, we prefer the following inequality, which follows directly from
Lemma~\ref{lemma:varianceAwareEmpBound} by applying the triangle inequality and a union bound over both arms (which replaces the factor $4$ by $8$ inside the logarithm).

\begin{lemma}[Variance-aware empirical bound for the distance of two arms]\label{lemma:varaw}
    Under the same assumptions as Lemma~\ref{lemma:varianceAwareEmpBound}
    \begin{equation*}
        \forall \delta' \in (0, 1] \colon \P{{\left|\empdist-\norm{\mu_j - \mu_i}  \right| \leq \mathcal{B}}} \geq 1 - \delta' \enspace,
    \end{equation*}
    for
    \begin{equation*}
        \mathcal{B} \coloneqq \left(\sqrt{\hat{\mathcal{V}}_i} + \sqrt{\hat{\mathcal{V}}_j}\right)\sqrt{2\frac{\log\frac{8}{\delta'}}{n}} + \frac{32}{3}\sqrt{\tilde{g}}\frac{\log\frac{8}{\delta'}}{n} \enspace.
    \end{equation*}
\end{lemma}

\begin{lemma}[Thm. 1 in \citep{wolfer2025variance}]\label{lemma:tolstikhin}
    Let $g$ be a continuous, positive definite and bounded kernel with supremum $\bar{g}$.
    Assume arm $i \in [N]$ was sampled $n$ times to calculate the empirical KME $\hat{\mu}_i$.
    Then it holds that for all $\delta \in (0, 1]$
    \begin{equation*}
        \P{{\Big\lvert\empdist-\norm{\mu_j - \mu_i}\Big\rvert \leq   \mathcal{B}_*}} \geq 1 - \delta \enspace, 
    \end{equation*}
for 
\begin{equation*}
        \mathcal{B_*} \coloneqq \left(\sqrt{\mathcal{V}_i^*} + \sqrt{\mathcal{V}_j^*}\right)\sqrt{2\frac{\log\frac{4}{\delta}}{n}} + \frac{8}{3}\sqrt{\bar{g}}\frac{\log\frac{4}{\delta}}{n} \enspace. 
\end{equation*}
\end{lemma}

\begin{lemma}[Bound for empirical Variance \citep{wolfer2025variance}]\label{lemma:empVarBound}
    Let $g$ be a continuous, positive definite, characteristic, translation invariant, bounded kernel.
    Assume arm $i \in [N]$ was sampled $n \in \Nnums$ times to calculate the empirical variance $\hat{\mathcal{V}}_i$.
    Let $b \in \left\{-1, 1\right\}$
    Then it holds that
    \begin{equation*}
        \forall \delta \in (0, 1] \colon \P{{b \left[\sqrt{\empvar{i}} - \sqrt{\mathcal{V}_i^*}\right] \leq 2 \sqrt{\frac{2 \tilde{g}\log\frac{1}{\delta}}{n}}}} \geq 1 - \delta \enspace.
    \end{equation*}
\end{lemma}

\section{Implementation details, calibration and ablation studies}\label{app:exp}

All experiments were run on a MacOS desktop with a 6-core (12-thread) 3.2 GHz Intel Core i7 processor and 16 GB of RAM, using Python 3.13 and CPU only. The repetitions were parallelized across either 5 or 10 processes. 
In particular, the limited memory forced the restriction on the allowable budget of Experiment 3 (many-cluster setting). 
As noted in the main paper, for Experiments 1 and 2, an upper bound on the budget is set for all algorithms, computed as $\tau=\frac{200N}{s_*^2}\log(\frac{N}{\delta})$, which is a heuristic simplification of \eqref{eq:bound_budget} with much smaller constants.
The signal-to-noise ratio $s_*$, specified in \eqref{def:SNR}, cannot be analytically computed for all distributions and kernels. Hence, we estimate it using 3000 samples per cluster. The estimation of $s_*$ is done for each setting, independent of all algorithms, and the same budget is applied across the algorithms.

The implementation of the algorithms \kabc, \kabcsg, \acb, and \tasfw\ differs from the theoretical version proposed in Section \ref{sec:algorithm} or \citet{thuot2024active,yavas2025general}.
The difference lies entirely in controlling the stopping criteria so that practical implementations converge faster than the listed algorithms suggest. We list and briefly justify the changes:
\begin{itemize}
    \item \kabc: The variance-aware threshold in \eqref{def:threshold} is theoretically proven, but it is conservative in practice. In particular, a large $\mathcal{B}^{ij}$ leads to fewer connected components, resulting in more iterations (and budget) for \kabc. In the implementation, we use a threshold $c_{\kabc}\mathcal{B}^{ij}$, where $\mathcal{B}^{ij}$ is as given in \eqref{def:threshold}, and the calibration constant is $c_{\kabc}=0.103516$.
    \item \kabcsg: This case is identical to \kabc, except that the threshold $\mathcal{B}^{ij}$ is specified by the sub-Gaussian bound stated in Section \ref{sec:algorithm}. We replace the threshold with $c_{\kabcsg}\mathcal{B}^{ij}$, where the calibration constant is $c_{\kabcsg}=0.121094$.
    \item \acb: The algorithm, as listed in \citet{thuot2024active}, refers to equations with 6 unspecified constants. These constants affect the thresholds and sampling budget, and the theory assumes them to be large enough absolute constants. The paper does not implement the exact adaptive algorithm, and the numerical experiments use only \acbg, which has all constants specified. 
    In our implementation, we replace the six unspecified constants in \acb\ with a calibrated constant $c_{\acb}=0.658984$.
    \item \tasfw: The stopping criterion for \tasfw\ is specified in \citet[Eqns. (65)--(66)]{yavas2025general}, and the algorithm stops when the stopping statistic crosses a computed threshold $\beta$. In our implementation, we lower the threshold to $c_{\tasfw}\beta$ with the calibrated constant $c_{\tasfw}=0.24094$. During numerical studies, we observed that, without calibration, \tasfw\ never stopped within the specified budget cap (of 30000). Lowering the threshold drastically increased the converged runs of \tasfw\ and improved runtime without affecting performance.
\end{itemize}

\begin{figure}[t]
  \centering
  \includegraphics[width=\textwidth,clip=true,trim=0mm 0mm 0mm 0mm]{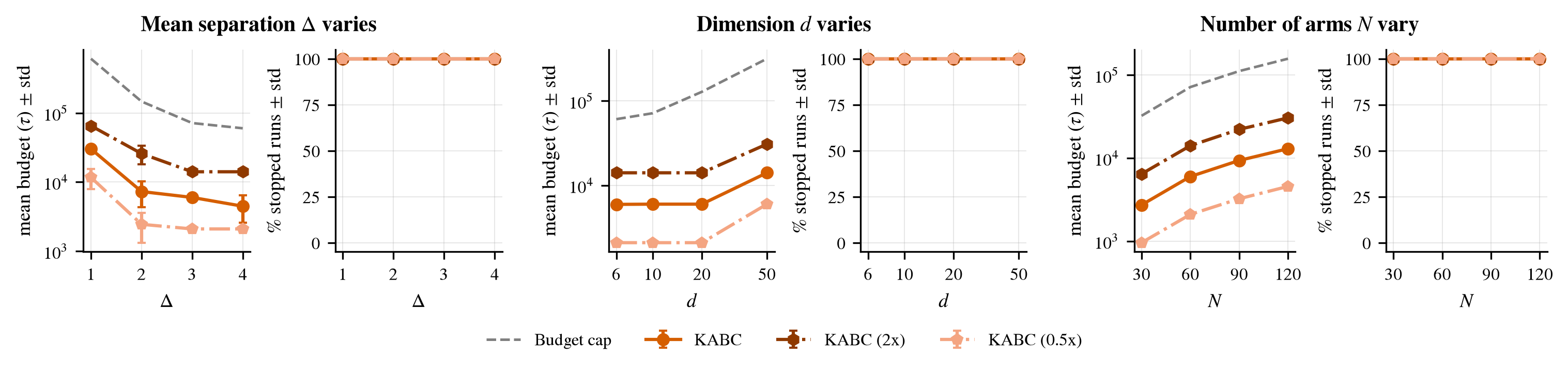}
  \vskip-1ex
  \caption{Performance of variants of \kabc, with different calibration constants in the case of Experiment 1 with Gaussian clusters. The plots reports quantities same as in Figure \ref{fig:exp-gmm}}
  \label{fig:exp-gmm-ablation}
\end{figure}

We note that the aforementioned calibration constants are kept fixed for every algorithm, independent of the data or experimental setting. Furthermore, the constants were not chosen arbitrarily but were selected using the following synthetic setups. For \kabc, \kabcsg, and \acb, we use the baseline configuration of Experiment 1 (Gaussian clusters) to identify constants over a grid so that, after rescaling the thresholds, the empirical error remains below 2.5\%. The same calibration constants were also used in all other experiments. Since \tasfw works only for distributions on a finite alphabet, we follow a similar protocol but use random pairs of distributions on $\{1,2,3,4,5\}$ in line with the settings of the MovieLens dataset.

It is natural to question the impact of the calibration constants. Figures \ref{fig:exp-gmm-ablation} and \ref{fig:exp-nonparam-ablation} show further results on the settings of Experiments 1 and 2, where we include two variants of \kabc, referred to as \kabc$(2\times)$ and \kabc$(0.5\times)$, where the calibration constant $c_{\kabc}$ is replaced by $2c_{\kabc}$ and $0.5c_{\kabc}$, respectively. Figure \ref{fig:exp-nonparam-ablation} additionally shows results on a 2-dimensional problem, where one distribution is a standard Gaussian, and the other corresponds to a ring with added noise. The results show that, in both experiments, changing the calibration constant does not impact the convergence of \kabc, but the mean budget obviously scales.

\begin{figure}[t]
  \centering
  \includegraphics[width=\textwidth,clip=true,trim=0mm 0mm 0mm 0mm]{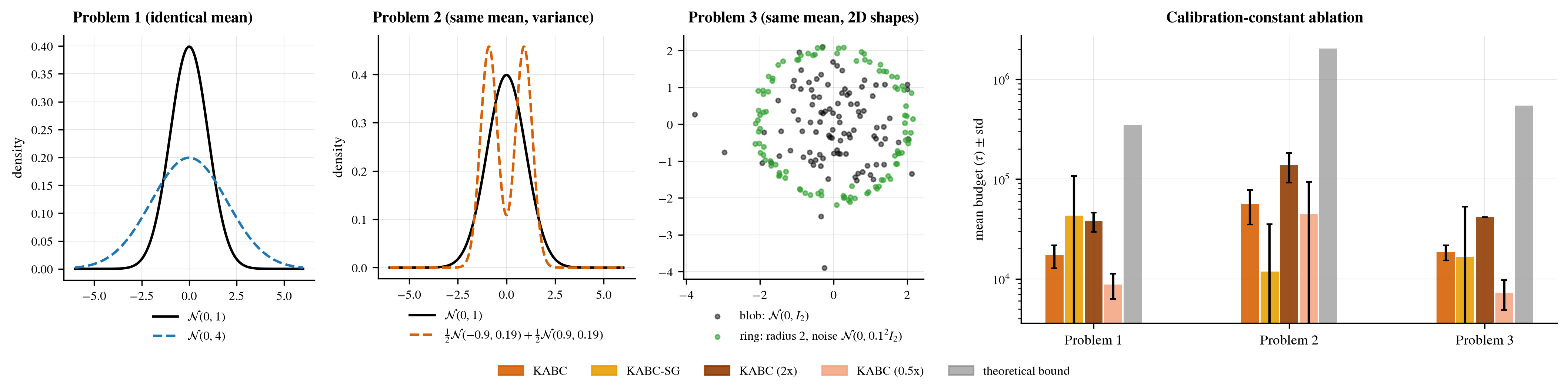}
  \vskip-1ex
  \caption{Performance of variants of \kabc, with different calibration constants in the case of Experiment 2 with nonparametric distribution. An additional 2D problem is also considered. The plots reports quantities same as in Figure \ref{fig:exp-nonparam}}
  \label{fig:exp-nonparam-ablation}
\end{figure}

\end{document}